\documentclass[smallextended]{svjour3}       

\smartqed  

\usepackage[table]{xcolor}
\colorlet{gr}{gray!45}
\colorlet{gr2}{gray!35}
\usepackage[utf8x]{inputenc}
\usepackage{xcolor}
\usepackage{booktabs}
\usepackage{listings}
\usepackage{amssymb,amsmath}
\usepackage{mathtools}
\usepackage[noend]{algpseudocode}
\usepackage{stmaryrd}
\usepackage{listings}
\usepackage{algorithm}
\usepackage{array}
\usepackage{multirow}
\usepackage{hhline}
\usepackage{textcomp}
\usepackage{pifont}
\usepackage{comment}
\usepackage{graphicx}
\usepackage[svgpath=figures/]{svg}
\usepackage[labelfont=bf]{caption}
\usepackage{tikz}
\usetikzlibrary{patterns}
\usetikzlibrary{calc}
\usetikzlibrary{decorations.pathreplacing}
\usetikzlibrary{shapes,arrows}
\usetikzlibrary{shapes.geometric}
\usetikzlibrary{math}

\usepackage{ifthen}
\usepackage{geometry}
\usepackage{marginnote}

\usepackage{ntheorem}

\usepackage{xcolor}
\usepackage{colortbl}
\usepackage{relsize}
\tikzset{fontscale/.style = {font=\relsize{#1}}}
\usepackage{marvosym}

\newcommand{\Author}{Manuel Schneckenreither
}
\newcommand{\Title}{Average Reward Adjusted Discounted Reinforcement Learning}
\newcommand{\Subtitle}{Near-Blackwell-Optimal Policies for Real-World Applications}
\newcommand{\Keywords}{machine learning, reinforcement learning, average reward, operations
  research, admission control.}
\newcommand{\ra}[1]{\renewcommand{\arraystretch}{#1}}
\newcommand{\ratab}{\ra{1.2}}
\newcommand\MS[2][r]{\ifx t#1 \textcolor{blue}{[\textbf{MS:} #2]} \else \begin{center}\textcolor{blue}{\textbf{MS:} #2} \end{center} \fi}
\usepackage{paper}

\journalname{Special Issue: Machine Learning and Combinatorial Optimization}

\usepackage[hidelinks,
plainpages=false,
pdftitle={\Title},
pdfauthor={\Author},
pdfsubject={\Title},
pdfkeywords={},
]{hyperref}
\usepackage{url}
\begin{document}

\title{\Title
}
\subtitle{\Subtitle}


\author{\Author}
\institute{M. Schneckenreither \at
  Department of Information Systems, Production and Logistics Management \\
  University of Innsbruck, Innsbruck 6020, Austria\\
  \email{manuel.schneckenreither@uibk.ac.at}
}
\date{Received: date / Accepted: date}

\maketitle

\begin{abstract}
  Although in recent years reinforcement learning has become very popular the number of successful
  applications to different kinds of operations research problems is rather scarce.
  Reinforcement learning is based on the well-studied dynamic programming technique and thus also
  aims at finding the best stationary policy for a given Markov Decision Process, but in contrast
  does not require any model knowledge.
  The policy is assessed solely on consecutive states (or state-action pairs), which are observed
  while an agent explores the solution space.
  The contributions of this paper are manifold. First we provide deep theoretical insights to the
  widely applied standard discounted reinforcement learning framework, which give rise to the
  understanding of why these algorithms are inappropriate when permanently provided with non-zero
  rewards, such as costs or profit.
  Second, we establish a novel near-Blackwell-optimal reinforcement learning algorithm. In contrary
  to former method it assesses the average reward per step separately and thus prevents the
  incautious combination of different types of state values. Thereby, the Laurent Series expansion
  of the discounted state values forms the foundation for this development and also provides the
  connection between the two approaches.
  Finally, we prove the viability of our algorithm on a challenging problem set, which includes a
  well-studied M/M/1 admission control queuing system. In contrast to standard discounted
  reinforcement learning our algorithm infers the optimal policy on all tested problems.
  The insights are that in the operations research domain machine learning techniques have to be
  adapted and advanced to successfully apply these methods in our settings.%
  \keywords{\Keywords}
  %
\end{abstract}
\pagebreak
\section{Introduction}%
\label{sec:Introduction}


Artificial Intelligence, or more precisely machine learning (ML), has been rising over the last
decades and became a very popular topic in science and industry. It is the field that tries to mimic
human learning capabilities using computer algorithms. This is done by mathematically modeling the
learning process by finding a function that matches the desired outcome according to a given input.
Machine learning itself can be divided into the categories of supervised learning, unsupervised
learning and reinforcement learning~\cite[p.1ff]{sutton1998introduction}, where we focus on
reinforcement learning in this paper.
Despite astonishing results in the application of reinforcement learning to game-like domains, for
instance
chess~\cite{Silver17_MasteringChessAndShogiBySelfPlayWithAGeneralReinforcementLearningAlgorithm},
Go~\cite{Silver16_MasteringTheGameOfGoWithDeepNeuralNetworksAndTreeSearch}, and various Atari
games~\cite{Mnih15_HumanlevelControlThroughDeepReinforcementLearningb}, no such success was reported
for non-episodic operations research problems yet.
We conjecture that this lack of reinforcement learning applications is due to the reason that the
widely applied standard discounted reinforcement learning framework is inappropriate. To overcome
this issue, we present a discounted reinforcement learning variant which is able to deduce
near-(Blackwell-)optimal policies, i.e.\@ it performs better for real-world economic problem
structures, such as ones often found in the area of operations research.

Reinforcement Learning (RL), which is mathematically based on dynamic programming, is similar to
supervised learning, but differs in the way that input-output pairs are actually never presented to
the algorithm, but an oracle function rewards actions taken by an
agent~\cite[p.2]{sutton1998introduction}.
In contrast to dynamic programming, RL has the advantages that (i) the problem space is explored by
an agent and thus only expectantly interesting parts need to be assessed and (ii) a comprehensive
knowledge of the underlying model becomes unnecessary as the states are evaluated by consecutively
observed states solely.
Like in operations research the goal is to optimise some measure, termed \textit{reward} in RL, by finding
the best function for the underlying problem.
The basic idea of RL is quite simple. An agent explores the solution space by observing the current
system state and taking an action which makes the system traverse to the next state. Every time the
agent chooses an action a reward is generated by the (oracle) reward function. The agent tries to
learn what actions to take to maximise the reward over time, i.e.\@ not only to maximise over the
actions for the current state, but it also respects possible stochastic transitions and reward values
of future states. RL is most often applied to control problems, games and other sequential decision
making tasks~\cite{sutton1998introduction}, where in almost all steps the reward function returns
\(0\). However, in real-world (economic) applications, where success is usually measured in terms of
profit or costs, the function intuitively returns a non-zero reward in almost all steps as it
continuously reports the actions results. Therefore, we focus on problems that produce an average
reward per step that cannot be approximated by \(0\).
%
%
Application areas of such a model could be an agent that periodically decides on buying and selling
instruments at the stock market, an agent performing daily replenishment decisions or many other
decision problems that arise in hierarchical supply chain and production planning systems, which are
strongly capacity-oriented~\cite{zijm2000towards} and with aggregated feedback (e.g.,
see~\cite{rohde2004hierarchical,hax1973hierarchical}).

Only a very limited number of operations research optimisation papers that are using RL and imposing
this structure are available.
Schneckenreither and Haeussler~\cite{schneckenreither2018reinforcement} presented an
RL algorithm which optimises order release decisions in production control environments. They use
sophisticated additions to allow the agent to link the actions to the actual rewards. The results
were compared to static order release measures only, which were outperformed.
%
Gijsbrechts et al.~\cite{gijsbrechts2018can} apply RL on a dual sourcing problem for replenishment
of a distribution center using rail or road, and compare their results to optimal solutions if
tractable or otherwise established heuristics. They found that hyperparameter\footnote{In machine
  learning parameters to be set before the start of the experiment are called hyperparameters.}
tuning is effort-intensive and the resulting policies are often not optimal, especially for larger
problems.
%
Balaji et al.~\cite{balaji2019orl} use out-of-the-box RL algorithms with simple 2 layer
neural networks to tackle stochastic bin packing, newsvendor and vehicle routing problems (VRP). The
VRP is a generalised travelling salesman problem (TSP) where one or more vehicles have to visit nodes in a
graph. They report to sometimes beat the benchmarks and find sensible solutions.
%
The capacitated VRP is also tackled with RL by Nazari et al.~\cite{nazari2018reinforcement}. They
minimise the total route length and compare the results to optimal solutions for small instances.
Although the optimum is not reached, one instance, which keeps track of the most probable paths,
performs better than well established heuristics. Also for larger problem sizes this technique seems
to outperforms the other tested methods.
Vera and Abad~\cite{vera2019deep} extend this model to a multi-agent algorithm and thus
tackle the capacitated multi-vehicle routing problem with a fixed fleet size. They also report
better results compared to the heuristics, especially for large problem sizes, but in contrast to
\cite{nazari2018reinforcement} are outperformed by Google's OR-Tools\footnote{See
  \url{https://developers.google.com/optimization}.}.
These applications, like other TSP applications~\cite[e.g.]{bello2016neural,kool2018attention}, of
RL to the VRP, except for~\cite{balaji2019orl}, however, calculate the reward according to the
length to the finished route. Thus, the average reward for long routes can be approximated with
\(0\) as the decision problem is episodic and therefore perform well.
%
%
%
%
%
%
Finally there are applications to the beer distribution
game~\cite{chaharsooghi2008reinforcement,oroojlooyjadid2017deep}, however the problem size and
therefore its complexity imposed by the beer game is rather small.
All of the above cited papers use highly sophisticated methods to tackle rather small problem sizes
or result in far-from-optimal solutions. One reason for this is that they utilise algorithms of the
widely applied standard discounted RL framework, which are usually evaluated on games. However,
games are structured to have a terminal state describing victory or defeat,
which is reported as positive or
negative reward to the agent, while any other preceding decision returns reward \(0\).
This does not comply with most of the above cited non-episodic optimisation problems.

Although some of these applications tackle intractable problems, in this paper we are concerned with
smaller problem sizes to be able to investigate the underlying mathematical issues.
A similar approach was taken by Mahadevan. In a series of papers the authors investigate average
reward RL\@. In average reward RL no discount factor is used but rather nested constraint problems
are approximated iteratively.
In~\cite{Mahadevan96_AverageRewardReinforcementLearningFoundationsAlgorithmsAndEmpiricalResults}
they establish mainly foundations and examine R-Learning, an average reward RL algorithm.
R-learning was proposed by Schwartz in~\cite{schwartz1993reinforcement} and is similar to our
approach, but less sophisticated and therefore unable to produce near-optimal-policies.
Then in~\cite{Mahadevan96_OptimalityCriteriaInReinforcementLearning} RL optimality criteria are
thoroughly discussed, while in the
papers~\cite{Mahadevan96_AnAveragerewardReinforcementLearningAlgorithmForComputingBiasoptimalPolicies,Mahadevan96_SensitiveDiscountOptimalityUnifyingDiscountedAndAverageRewardReinforcementLearning}
they present model-based and model-free bias-optimal algorithms. Similarly, Tadepalli and
Ok~\cite{Tadepalli98_ModelbasedAverageRewardReinforcementLearning} present an average reward RL
algorithm called H-Learning, which under certain assumptions finds bias-optimal values. Furthermore,
they apply it to simulated automatic guided vehicle scheduling.
Later Mahadevan et al.~\cite{mahadevan1997self} lift the model-free average reward RL algorithm to
continuous-time semi-Markov Decision Processes (semi-MDPs), which handle non-periodic decision
problems, and apply it to an inventory problem consisting of a single machine. Then the algorithm
of~\cite{mahadevan1997self}, called SMART, was applied on the optimisation of transfer lines using a
hierarchical approach~\cite{mahadevan1998optimizing} and preventive maintenance of a production
inventory system~\cite{das1999solving}. Although the results are promising, the adaption to
continuous-time problems eases the complexity for most applications. However, in practise usually
decision have to be made on a daily basis~\cite{enns2004work}. Therefore, we refrain from this
adaption and concentrate on standard MDPs only. Furthermore, continuous-time problems can be
converted through uniformisation into equivalent discrete time instances
(see~\cite{Puterman94,bertsekas1995dynamic}).

The reason for the small number of RL applications to non-episodic operations research problems and
the fact that none of the available applications report tremendous success, as in game-like areas,
can be explained by the average reward per step that the agent receives.
The average reward per step is usually by far the greatest part when the discounted state
values\footnote{Note that in contrary to economics where discounting is often motivated by interest
  rates in discounted RL the motivation of discounting future rewards origins from the fact that the
  ``infinite sum has a finite value as long as the reward sequence [..] is
  bounded''~\cite[p.59]{sutton1998introduction}.} are decomposed into its sub-components.
This results from the fact, that in contrast to the other parts, the average reward contributes to
the learned state-values in an exponentially up-scaled fashion. However, for most applications the
average reward is equal among all states and thus the agent has issues in choosing the best action.
Additionally, iteratively shifting all state values to the base imposed by the exponentially
up-scaled average reward easily causes the learning process to fail, finding itself in sub-optimal
maxima and with that complicating the hyperparameterisation process tremendously.
%
On the opposite average reward RL is cumbersome and computationally expensive,
as nested constraints have to be solved.
Therefore, we contribute to the research streams of discounted and average reward RL, by combining
the best things of both worlds, that is, the stability and simplicity of standard discounted RL and
the idea of assessing the average reward separately and aiming for broader optimality criteria of
average reward RL with this work. Additionally, we contribute to the operations research domain,
which was the origin of the motivation for this work, by providing a new machine learning algorithm
and applying it to a well studied M/M/1 admission control queuing system.
%
%
%
In summary the contribution is as follows.

\begin{itemize}
\item First we analyse and illustrate why standard discounted reinforcement learning is
  inappropriate for real-world applications, which usually impose an average reward that cannot be
  approximated with \(0\),
\item Second, we establish a novel near-Blackwell-optimal reinforcement learning algorithm and
  analytically prove its optimality, and
\item Third, show the viability of the algorithms by experimentally applying them to three problem
  specifications, one of which is a well-studied M/M/1 admission control queuing system.
\end{itemize}

The rest of the paper is structured as follows. The next section provides an overview of the
proposed method. Section~\ref{sec:The_Laurent_Series_Expansion_of_Discounted_State_Values}
introduces discounted RL and average reward RL, and provides the linkage between the two
frameworks via the Laurent Series Expansion of discounted state values. The insights gained of the
expansion form the foundations for the developed average reward adjusted discounted RL method
established in Section~\ref{sec:Average_Reward_Adjusted_Discounted_Reinforcement_Learning}.
Section~\ref{sec:Experimental_Evaluation} proves the viability of the algorithm, while
Section~\ref{sec:Conclusion} concludes the paper.

\section{Overview of Average Reward Adjusted Discounted Reinforcement Learning}%
\label{sec:The_Need_of_More_Refined_Reinforcement_Learning_Agents_in_Real-World_Applications}

Before establishing the mathematical details of the method in the next sections, in this section we
give a high-level overview of the newly developed algorithm by providing the main concepts and ideas
using an example.

Although very convenient the evaluation of newly developed RL algorithms on game-like environments,
for instance chess or Atari games, brings major issues once these algorithms are applied for
optimisation in structurally different real-world scenarios as found in operations research.
The most intuitive goal in such economic optimisations problems is to maximise for expected profit
(minimize expected costs) at every decision, that is, always taking the actions that expectantly
result in the greatest profit.
I.e.\@ when modelled as reinforcement learning process the reward function must return the
accumulated profit between consecutive decisions. However, reward functions of games are usually
designed to return a non-zero reward only after the agent reached the state describing victory or
defeat. Therefore, currently newly designed RL algorithms are evaluated with problems that impose an
average reward per step of approximately \(0\). This leads to the fact, that when discounted RL
techniques are applied to real-world scenarios, they often perform poorly.
As we will see the average reward per step plays a crucial role in the performance of reinforcement
learning algorithms.

Therefore, our approach in this work is to separately learn the average reward value and rather use
\textit{average reward adjusted discounted state values} to be able to selectively choose the best
actions. By using a sufficiently large discount factor the agent reduces the set of best actions to
the ones that are (bias-)optimal. Additionally, we approximate the average reward adjusted
discounted state values with a smaller discount factor, which allows us to choose actions
near-(Blackwell-)optimal.

To clarify consider the Markov Decision Process (MDP) in Figure~\ref{fig:three-states} which
consists of two possible definite policies with the only action choice in state \(1\). Both policies
\(\pol_{A}\) (going left in 1) and \(\pol_{B}\) (going right) result in the same average reward
\(\avgrew^{\pol_{A}} = \avgrew^{\pol_{B}} = 1\).
However, only the A-loop is (bias-)optimal, as the actions with non-zero rewards are selected
earlier. E.g.\@ consider starting in state 1. Under the policy \(\pol_{A}\) which takes the A-loop the
reward sequence is \((2,0,2,0,\ldots)\), while for the other policy \(\pol_{B}\) it is
\((0,2,0,2,\ldots)\).
Our algorithm infers i) the average reward based on the state values of consecutively observed
states and the returned reward, ii) approximates so called bias values, and iii) uses a smaller
discount factor to infer the greatest error term values which only exist as the discount factor is
strictly less than \(1\). Intuitively bias values are the additional reward received when starting
in a specific state, and the error term incorporates the number of steps until the reward is
collected.

With a discount factor of \(0.999\) our implementation automatically infers policy \(\pol_{A}\) and
thus produces average reward adjusted discounted state values
\(\X_{0.999}^{\pol_{A}}((A,l)) = 0.493\) for going left (\(l\)) in state \(A\) and
\(\X_{0.999}^{\pol_{A}}((A,r)) = 0.492\) for going right (\(r\)) in under \(10k\) iterations. Here
we use state-action pairs as function parameters. The actual bias values, analytically inferred, are
\(\V^{\pol_{A}}((A,l)) = \V^{\pol_{A}}((A,r)) = 0.5\), where the differences to the estimated values
are due to the error term.
%
%
Note that for policy \(\pol_{B}\) the state values are
\(\V^{\pol_{B}}((A,l)) = \V^{\pol_{B}}((A,r)) = -0.5\). Therefore, \(\pol_{A}\) is preferable.
However, after the decision for \(\pol_{A}\) is made, doing loop B once becomes attractive as well,
as the same amount of rewards are collected. This makes the bias values being equal under the given
policy.
%
Therefore, our algorithm has a further decision layer using an approximation of average adjusted
discounted state values with a smaller discount factor. As the error term is increased when
the discount factor decreases, the agent chooses the action that maximises the error term. The inferred values for a discount factor of \(0.8\) are
\(\X_{0.8}^{\pol_{A}}((A,l)) = 0.555\) and \(\X_{0.8}^{\pol_{A}}((A,r)) = 0.145\), and thus action
\(l\) is preferred over action \(r\).

\begin{figure}[t!]
  \centering
  \begin{tikzpicture}[thin, scale=0.75]
  \draw (-3,0) node(0) [circle,draw,minimum size=25] {\footnotesize \(0\)};
  \draw (0,0)  node(1) [circle,draw,minimum size=25] {\footnotesize \(1\)};
  \draw (3,0)  node(2) [circle,draw,minimum size=25] {\footnotesize \(2\)};

  \path[thin, ->, bend right, >=stealth] (1) edge[above] node {\footnotesize\(2\)} (0);
  \path[thin, ->, bend right, >=stealth] (2) edge[above] node {\footnotesize\(2\)} (1);
  \path[thin, ->, bend right, >=stealth] (0) edge[below] node {\footnotesize\(0\)} (1);
  \path[thin, ->, bend right, >=stealth] (1) edge[below] node {\footnotesize\(0\)} (2);

  \draw (-1.5,0) node[] { A };
  \draw (1.5,0) node[] { B };
\end{tikzpicture}

  \caption{\label{fig:three-states} A MDP with the only action choice in state \(1\) producing
    policies \(\pol_{A}\) (going left in 1) and \(\pol_{B}\) (going right) (adapted
    from~\cite{Mahadevan96_AverageRewardReinforcementLearningFoundationsAlgorithmsAndEmpiricalResults})}
\end{figure}
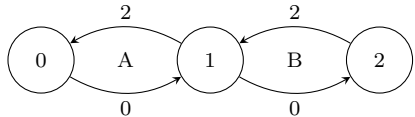

In the sequel we establish the method, provide the algorithm and present an optimality analysis
thereof. I.e.\@ we show that the presented algorithm is capable of inferring, so called,
Blackwell-optimal policies~\cite{Blackwell62}. Blackwell-optimal policies are policies that maximise
the average reward, the bias value and the error term, in this specific order.
Near-Blackwell-optimal algorithms infer bias-optimal policies for any MDP and for given MDPs can be
configured to infer Blackwell-optimal policies, where inaccuracies due to the limitations of
floating point representations of modern computer systems are neglected.

\section{The Laurent Series Expansion of Discounted State Values}%
\label{sec:The_Laurent_Series_Expansion_of_Discounted_State_Values}


This section briefly introduces the needed formalism, then investigates the discounted RL framework
and provides the Laurent series expansion of its state values. The Laurent series expansion plays a
crucial role in the development of the newly established algorithm.

Like Miller and Veinott~\cite{MillerVeinott1969} we are considering problems that are observed in a
sequence of points in time labeled \(1,2,\ldots\) and can be modelled using a finite set of states
\(\States\), labelled \(1,2,\ldots,\size{\States}\), where the size \(\size{\States}\) is the number
of elements in \(\States\). At each point $t$ in time the system is in a state
\(s_{t} \in \States\). Further, by choosing an action $a_{t}$ of a finite set of possible actions
\(A_{s}\) the system returns a reward $r_{t} = r(s_{t}, a_{t})$ and transitions to another state
\(s_{t+1} \in \States\) at time \(t+1\) with conditional probability
\(p(s_{t+1}, r_{t} \mid s_{t}, a_{t})\). That is, we assume that reaching state \(s_{t+1}\) from
state \(s_{t}\) with reward \(r_{t}\) depends solely on the previous state \(s_{t}\) and chosen
action \(a_{t}\). In other words, we expect the system to possess the Markov property
\cite[p.63]{sutton1998introduction}. RL processes that possess the Markov property are referred to
as Markov decision processes (MDPs)~\cite[p.66]{sutton1998introduction}. A MDP is called
\textit{episodic} if it includes terminal states, or \textit{non-episodic}
otherwise~\cite[p.58]{sutton1998introduction}. \textit{Terminal states} are absorbing states and are followed
by a reset of the system, which starts a new episode.

Thus, the action space is defined as \(F = \times_{s=1}^{\size{\States}} A_{s}\), where \(A_{s}\) is
a finite set of possible actions. A \emph{policy} is a sequence \(\pol = (f_{1},f_{2},\ldots)\) of
elements \(f_{t} \in F\). Using the policy \(\pol\) means that if the system is in state \(s\) at
time \(t\) the action \(f_{t}(s)\), i.e. the \(s\)-th component of \(f_{t}\), is chosen. A
stationary policy \(\pol = (f,f,\ldots)\) does not depend on time. In the sequel we are concerned
with stationary policies only.
An \textit{ergodic} MDP consists of a single set of recurrent states under all stationary policies,
that is, all states are revisited with probability
\(1\)~\cite{Mahadevan96_AverageRewardReinforcementLearningFoundationsAlgorithmsAndEmpiricalResults}.
A MDP is termed \textit{unichain} if under all stationary policies the transition matrix contains a
single set of recurrent states and a possible empty set of transient states. A MDP is
\textit{multichain} if there exists a policy with at least two recurrent
classes~\cite{Mahadevan96_AverageRewardReinforcementLearningFoundationsAlgorithmsAndEmpiricalResults}.
Finally, a state is termed \textit{periodic} if the greatest common divisor of all path lengths to
itself is greater than \(1\). Otherwise, it is called aperiodic.
%
%
The goal in RL is to find the optimal stationary policy \(\pol^{\star}\) for
the underlying MDP problem defined by the state and action spaces, as well as the reward function.

\subsection{Discounted Reinforcement Learning}
\label{subsec:Discounted_Reinforcement_Learning}

In the standard discounted framework the value of a state \(V_{\gamma}^{\pol_{\gamma}}(s)\) is
defined as the expected discounted sum of rewards under the stationary policy \(\pol_{\gamma}\) when
starting in state \(s\). That is,
\begin{align*}
  V_{\gamma}^{\pol_{\gamma}}(s) = \lim_{N \to \infty} E[\sum_{t=0}^{N-1} \gamma^{t} R_{t}^{\pol_{\gamma}}(s)]\tcom
\end{align*}

where \(0 \leqslant \gamma < 1\) is the discount factor and
\(R_{t}^{\pol_{\gamma}}(s) = \E_{\pol_{\gamma}}[ r(s_{t},a_{t}) \mid s_{t} = s, a_{t} = a]\) 
the reward received at time \(t\) upon starting in state \(s\) by following policy \(\pol_{\gamma}\)
\cite[e.g.]{Mahadevan96_AverageRewardReinforcementLearningFoundationsAlgorithmsAndEmpiricalResults}.
The aim is to find an optimal policy \(\polopt_{\gamma}\), which when followed, maximises the state
value for all states \(s\) as compared to any other policy \(\pol_{\gamma}\):
\begin{align*}
  V_{\gamma}^{\polopt_{\gamma}} - V_{\gamma}^{\pol_{\gamma}} \geqslant 0\tpkt
\end{align*}

This criteria is usually referred to as \textit{discounted-optimality} (or
\(\mathit{\gamma}\)\textit{-optimality}) as the discount factor \(\gamma\) is
fixed~\cite{Mahadevan96_OptimalityCriteriaInReinforcementLearning}.
Note that most works omit the index \(\gamma\) in the policies \(\pol_{\gamma}, \polopt_{\gamma}\)
and thus incorrectly indicate that \(\polopt_{\gamma} = \polopt\), where \(\polopt\) is the optimal
policy for the underlying problem. For the rest of the paper we follow this convention and drop the
index \(\gamma\) of the policy for the discounted state value and the reward, thus
\(V_{\gamma}^{\pol_{\gamma}}(s) = V_{\gamma}^{\pol}(s)\), and
\(R_{t}^{\pol_{\gamma}}(s) = R_{t}^{\pol}(s)\).

This also means that the actual value set for \(\gamma\) defines the policy which is optimised for.
For instance, Figure~\ref{fig:printer} depicts a MDP with a single action choice in state \(1\).
Rewards of \(5\) or \(10\) respectively, are received once upon traversing from state \(5\) or
\(10'\) to state \(1\). In all other cases the reward is \(0\).
The only action choice is in state \(1\), in which the agent can choose between doing the
printer-loop or the mail-loop. Observe that the average reward received per step equals \(1\) for
the printer-loop and \(2\) for the mail-loop.
Thus, the Blackwell-optimal policy is to choose the mail-loop, as it maximises the returned reward.
However, if \(\gamma < 3^{-\frac{1}{5}} \approx 0.8027\) an agent using standard discounted
reinforcement learning selects the printer loop.
%
%
Schartz~\cite{schwartz1993reinforcement} shows that for arbitrary large \(\gamma < 1\) it is
possible that standard discounted RL fails in finding the optimal policy, while Zhang and
Dietterich~\cite{zhang1995reinforcement} as well as Boyan~\cite{boyan1995generalization} use RL on
combinatorial optimization problems for which the choice of discount factor is crucial.

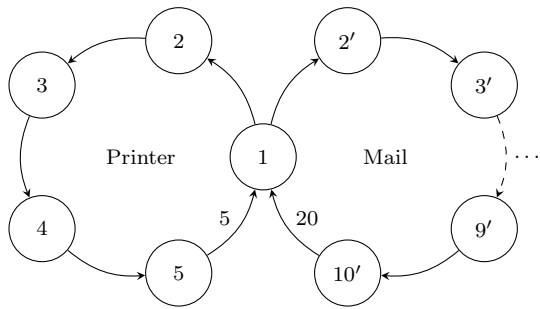
\begin{figure}[t!]
  \centering
  \begin{tikzpicture}[thin, scale=0.65]
  \foreach \x in {1,...,5}{
    \draw ({cos(\x*72-72)*2.5},{sin(\x*72-72)*2.5}) node(\x) [circle,draw,minimum size=25] {\footnotesize \(\x\)};
  };
  \foreach[evaluate={
    \y=int(\x+5);
  }]  \x in {2,...,3}{
    \draw ({-cos(\x*72-72)*2.5+5},{sin(\x*72-72)*2.5}) node(\y) [circle,draw,minimum size=25]
    {\footnotesize \(\x'\)};
  };
  \foreach [evaluate={
    \y=int(\x+5);
    \z=int(\x+5);
  }] \x in {4,...,5}{
    \draw ({-cos(\x*72-72)*2.5+5},{sin(\x*72-72)*2.5}) node(\z) [circle,draw,minimum size=25]
    {\footnotesize \(\y'\)};
  };

  \newcommand\bend{22.5}
  \path[thin, ->, bend right=\bend, >=stealth] (1) edge[below left] node {\footnotesize\(\)} (2);
  \path[thin, ->, bend right=\bend, >=stealth] (2) edge[above]      node {\footnotesize\(\)} (3);
  \path[thin, ->, bend right=\bend, >=stealth] (3) edge[left ]      node {\footnotesize\(\)} (4);
  \path[thin, ->, bend right=\bend, >=stealth] (4) edge[above]      node {\footnotesize\(\)} (5);
  \path[thin, ->, bend right=\bend, >=stealth] (5) edge[above left] node[yshift=-2] {\footnotesize\(5\)} (1);

  \path[thin, ->, bend left=\bend, >=stealth]         (1)  edge[below right] node {\footnotesize\(\)} (7);
  \path[thin, ->, bend left=\bend, >=stealth]         (7)  edge[above right] node {\footnotesize\(\)} (8);
  \path[thin, ->, dashed, bend left=\bend, >=stealth] (8)  edge[right]       node {\(\ldots\)} (9) ;
  \path[thin, ->, bend left=\bend, >=stealth]         (9)  edge[above]       node {\footnotesize\(\)} (10);
  \path[thin, ->, bend left=\bend, >=stealth]         (10) edge[above right] node[yshift=-2] {\footnotesize\(20\)} (1);

  \draw (0,0) node[] { Printer };
  \draw (5,0) node[] { Mail };

\end{tikzpicture}

  \caption{\label{fig:printer} A MDP with two different definite policies (adapted and
    corrected\protect\footnotemark{} from
    Mahadevan~\cite{Mahadevan96_OptimalityCriteriaInReinforcementLearning})}
\end{figure}

\footnotetext{In~\cite{Mahadevan96_OptimalityCriteriaInReinforcementLearning} they claim that only
  for \(\gamma < 0.75\) the policy is sub-optimal.}

The idea behind the discount factor in standard discounted reinforcement learning is to prevent
infinite state values~\cite[p.59]{sutton1998introduction}. However,
Mahadevan~\cite{Mahadevan96_OptimalityCriteriaInReinforcementLearning,Mahadevan96_AverageRewardReinforcementLearningFoundationsAlgorithmsAndEmpiricalResults}
refers to the discounting approach as unsafe, as it encourages the agent to aim for short-term gains
over long-term benefits. This results from the fact that the impact of an action choice with
long-term reward decreases exponentially with time~\cite{ok1996auto}. 
%
Besides bounding the state values another main idea behind the \(\gamma\)-parameter is to be able to
specify a balance between short-term (low \(\gamma\)-values) and long-term (high \(\gamma\)-values)
reward objectives
~\cite{schwartz1993reinforcement}.
But what seems to be an advantage rather becomes a disadvantage, as in almost all decision problems
the aim actually is to perform well over time, i.e.\@ to collect as much reward as possible over a
presumably infinite time horizon.
In terms of reward this means to seek for average reward maximising policies before more selectively
choosing actions, cf.\@ the example in Figure~\ref{fig:printer}. Therefore, in almost all RL studies
the discount factor \(\gamma\) is set to a value very close to (but strictly less than) \(1\), for
instance to \(0.99\)~\cite[e.g.]{mnih2015human,mnih2016asynchronous,Lillicrap15}. However, the issue
is that the state values increase exponentially as the average reward is multiplied by
\(1/(1-\gamma)\), while the bias value and error term are not. We will refer to the portion induced
by the exponentially scaled average reward as \textit{base level imposed by the average reward}.

\begin{example}
  To clarify, consider the gridworld MDP given in Figure~\ref{fig:grid} with states
  \(\mathcal{S} = \{ (x,y) \mid x,y \in \{0,1\} \}\), where \((0,0)\) is the goal state after which the agent is
  placed uniformly on the grid by using the only available action \textsf{random} (see left side).
  In all other states the agent can choose among the actions \textsf{up}, \textsf{right},
  \textsf{down}, \textsf{left}, which move the agent in the expected direction, except when moving
  out of the grid. In such cases the state is unchanged. The reward function is given on the right.
  Action \textsf{random} in state \((0,0)\) returns reward \(10\), while taking any other action
  provides a uniformly distributed reward of \(\Unif(0,8)\), where moving out of the grid adds a
  punishment of \(-1\).
  It is obvious that reaching state \((0,0)\) with as little steps as possible is optimal.
  As by definition the average reward for unichain MDPs is equal among all
  states~\cite{Mahadevan96_AverageRewardReinforcementLearningFoundationsAlgorithmsAndEmpiricalResults}.
  Thus the average reward for all states of the optimal (and greedy) policy is \(7\), which already
  imposes a base level imposed by the average reward of \(700\) in case of \(\gamma=0.99\).

  \begin{figure}[t!]
    \centering
    \begin{tikzpicture}[thin, scale=1.5]
  \newcommand\maxX{1}         
  \newcommand\maxY{1}         
  \newcommand\goalX{0}         
  \newcommand\goalY{0}         

  \tikzmath{\maxXP = \maxX + 1; \maxYP = \maxY + 1;};

  \foreach \x in {0,...,\maxXP} {
    \foreach \y in {0,...,\maxYP} {
      \draw (\x,0) -- (\x,\maxYP) node[] {};
      \draw (0,\y) -- (\maxXP,\y) node[] {};
    }
  }
  \foreach \x in {0,...,\maxX} {
    \foreach \y in {0,...,\maxY} {
      \tikzmath{\xP = int(\maxX - \x); \yP = int(\maxY - \y);};
      \draw (\x+0.5, \y+0.5) node[] {(\yP,\x)};
      \draw (\x+0.5, \y+0.9) node[] (\yP,\x,up)                  {\ifthenelse{\x=\goalX \AND \yP=\goalY}{}{\tiny up   }};
      \draw (\x+0.5, \y+0.1) node[] (\yP,\x,down)                {\ifthenelse{\x=\goalX \AND \yP=\goalY}{}{\tiny down }};
      \draw (\x+0.1, \y+0.5) node[rotate=90] (\yP,\x,left)       {\ifthenelse{\x=\goalX \AND \yP=\goalY}{}{\tiny left }};
      \draw (\x+0.9, \y+0.5) node[rotate=90] (\yP,\x,right)      {\ifthenelse{\x=\goalX \AND \yP=\goalY}{}{\tiny right}};
    }
  }
  \draw (\goalX+0.5, \maxY-\goalY+0.3) node[] (\goalX,\goalY,rand) {\tiny random};

  \draw (\maxXP/2, 0) node[below,minimum height=0.75cm] { \small State/Action Space};

\end{tikzpicture}
\begin{tikzpicture}[thin, scale=1.5]
  \newcommand\maxX{1}         
  \newcommand\maxY{1}         
  \newcommand\goalX{0}         
  \newcommand\goalY{0}         

  \tikzmath{\maxXP = \maxX + 1; \maxYP = \maxY + 1;};

  \foreach \x in {0,...,\maxXP} {
    \foreach \y in {0,...,\maxYP} {
      \draw (\x,0) -- (\x,\maxYP) node[] {};
      \draw (0,\y) -- (\maxXP,\y) node[] {};
    }
  }
  \foreach \x in {0,...,\maxX} {
    \foreach \y in {0,...,\maxY} {
      \tikzmath{\xP = int(\maxX - \x); \yP = int(\maxY - \y);};
      \draw (\x+0.5, \y+0.5) node[] {(\yP,\x)};
      \draw (\x+0.5, \y+0.9) node[] (\yP,\x,up)                  {\ifthenelse{\x=\goalX \AND \yP=\goalY}{}{\tiny \(\Unif(0,8)\)\ifthenelse{\yP=0    }{\(-1\)}{} }};
      \draw (\x+0.5, \y+0.1) node[] (\yP,\x,down)                {\ifthenelse{\x=\goalX \AND \yP=\goalY}{}{\tiny \(\Unif(0,8)\)\ifthenelse{\yP=\maxY}{\(-1\)}{} }};
      \draw (\x+0.1, \y+0.5) node[rotate=90] (\yP,\x,left)       {\ifthenelse{\x=\goalX \AND \yP=\goalY}{}{\tiny \(\Unif(0,8)\)\ifthenelse{\x=0     }{\(-1\)}{} }};
      \draw (\x+0.9, \y+0.5) node[rotate=90] (\yP,\x,right)      {\ifthenelse{\x=\goalX \AND \yP=\goalY}{}{\tiny \(\Unif(0,8)\)\ifthenelse{\x=\maxX }{\(-1\)}{} }};
    }
  }
  \draw (\goalX+0.5, \maxY-\goalY+0.3) node[] (\goalX,\goalY,rand) {\tiny 10};

  \draw (\maxXP/2, 0) node[below,minimum height=0.75cm] { \small Reward Function};

\end{tikzpicture}

    \caption{\label{fig:grid} A gridworld MDP with \(4\) states and a set of \(5\) actions
      and the reward function
    }
  \end{figure}
\end{example}

For episodic MDPs standard discounted RL is unable to evaluate the average reward of terminal states
and its predecessors correctly, as there is no consecutive state at the end of an episode, i.e.\@
for the gridworld example state \((0,0)\) is evaluated with value \(10\), when a new episode starts
upon reaching the goal.

Furthermore, for both, episodic and non-episodic MDPs, there is another major issue. As each state
is assessed separately, so is the average reward of that state. Thus, states that are visited more
often in the iterative process of shifting all state values to the base level imposed by the average
reward 
are evaluated with a higher average reward, which however, increases the likeliness that the agent
will visit the state again. This behaviour forms clusters with cyclic paths of states (e.g.\@ going
back and fourth between states \((1,0)\) and \((1,1)\)), which are visited more and more likely.
Even setting a very high exploration rate is usually no remedy by the same argument. Additionally,
recall that the average reward increases once the policy gets better through decreasing the
exploration rate. This information then needs to be traversed to all states, which however, requires
high exploration. This implies that finding the correct hyperparameter setting (e.g.\@ exploration
decay) needs a huge amount of effort and experience, which is exactly what was reported
in~\cite{gijsbrechts2018can}.
%
Regardless of that, this behaviour increases the number of required learning steps tremendously, as
all states have to be adapted every time a policy change imposes a change in the average reward (see
also~\cite{Mahadevan96_OptimalityCriteriaInReinforcementLearning,Mahadevan96_AverageRewardReinforcementLearningFoundationsAlgorithmsAndEmpiricalResults,schwartz1993reinforcement}).
Unfortunately this easily leads to aforementioned clusters.

It is obvious that optimal policies are very hard to obtain with all
these difficulties, especially as all these effects complicate the
parameterisation tremendously.
%
However, recall that in operations research and many other real-world
applications, success is constantly measured. It is easy to see that
this corresponds directly to the problem structure of the gridworld
example given in Figure~\ref{fig:grid}. Unfortunately this also means
that all these issues directly exist in such applications, which also
explains the small number of works combining RL
and operations research problems.

Therefore, in the sequel we present a more refined RL approach for operations research, which
overcomes these issues by separately assessing the average reward.
The relation of the computed values of these two approaches, that is, the aformentioned discussed
standard discounted RL technique and the in this paper developed average reward adjusted discounted
RL (\ARA{}) method, are described by following established Laurent series expansion.

\subsection{The Laurent Series Expansion of Discounted State Values}
\label{subsec:The_Laurent_Series_Expansion_of_Discounted_State_Values}

The Laurent series expansion of the discounted state values~\cite{MillerVeinott1969,Puterman94}
provide important insights by giving rise to basically three addends. For a given discount factor
\(\gamma\) the first addend is solely determined by the average reward\footnote{Please note that we
  assume a positive average reward and the objective to maximise the returned reward throughout this
  work.}, the second is the bias value and the third one, actually consisting of infinitely many
sub-terms, is the error term.
In the sequel we present the definitions of the average reward and bias value, before providing the
Laurent series expansion.

\begin{definition}
  Due to Howard~\cite{howard1960dynamic} for an aperiodic\footnote{In the periodic case the Cesaro
    limit of degree \(1\) is required to ensure stationary state transition probabilities and thus
    stationary values~\cite{Puterman94}. Therefore to ease readability we concentrate on unichain
    and aperiodic MDPs. However, the theory directly applies to period unichain MDPs by replacing
    the limits accordingly.} MDP the \textit{gain} or \textit{average reward} \(\avgrew^{\pol}(s)\)
  of a policy \(\pol\) and a starting state \(s\) is defined as
  \begin{align*}
    \avgrew^{\pol}(s) = \lim_{N \to \infty} \frac{\E [\sum_{t=0}^{N-1}R_{t}^{\pol}(s)]}{N}\tcom
  \end{align*}
  where \(R_{t}^{\pol}(s) = \E_{\pol}[ r(s_{t},a_{t}) \mid s_{t} = s, a_{t} = a]\) is the reward received at time \(t\), starting in state \(s\) and
  following policy \(\pol\).
\end{definition}

Clearly \(\avgrew^{\pol}(s)\) expresses the expected average reward received per action taken
when starting in state \(s\) and following policy \(\pol\). In the common case of unichain MDPs, in
which only a single set of recurrent states exists, the average reward \(\avgrew^{\pol}(s)\) is
equal for all states \(s\)
\cite{Mahadevan96_AverageRewardReinforcementLearningFoundationsAlgorithmsAndEmpiricalResults,Puterman94}.
Thus, in the sequel we may simply refer to it as \(\avgrew^{\pol}\).

\begin{definition}
  For an aperiodic MDP problem the average adjusted sum of rewards or
  \textit{bias value} is defined as
  \begin{align*}
    V^{\pol}(s) = \lim_{N \to \infty}{ \E [ \sum_{t=0}^{N-1}(R_{t}^{\pol}(s) - \avgrew^{\pol}(s) )]}\tcom
  \end{align*}
  where again \(R_{t}^{\pol}(s)\) is the reward received at time \(t\), starting in state \(s\) and
  following policy \(\pol\).
\end{definition}

Note that the bias values are bounded due to the subtraction of the
average reward. Thus the bias value can be seen as the rewards that additionally sum up in case the
process starts in state \(s\).

%

Finally, a state value \(V_{\gamma}^{\pol}(s)\) in standard discounted reinforcement learning can be
decomposed in its average reward, the bias value and an additional error term, which actually
consists of infinitely many subterms.

\begin{definition}
  Due to Miller and Veinott~\cite{MillerVeinott1969} the \textit{Laurent Series Expansion of the
    discounted state value} for a state \(s\), a discount factor \(\gamma\) and a policy \(\pol\) is
  given by
  \begin{align}
    \label{eq:laurent}
    V_{\gamma}^{\pol}(s) = \frac{\avgrew^{\pol}(s)}{1-\gamma} + V^{\pol}(s) + e_{\gamma}^{\pol}(s) \tcom
  \end{align}
  where Puterman~\cite[p.341]{Puterman94} shows that \(\lim_{\gamma \to 1} e_{\gamma}^{\pol}(s) = 0\).
\end{definition}

The error term \(e_{\gamma}^{\pol}(s)\) incorporates what amount of reward is collected in
combination with the number of steps until it is collected. The higher the discount factor the more
long-sighted the agent is.
Note how the first term depending on the average reward \(\avgrew^{\pol}(s)\) converges to
infinity as \(\gamma\) increases towards \(1\) and that the second addend does not depend on the
discount factor. If the average reward is non-zero and for large \(\gamma\)-values we can assume
\(\avgrew^{\pol}(s)/(1-\gamma) \gg V^{\pol}(s) + e_{\gamma}^{\pol}(s)\) which explains the behaviour
of the standard discounted RL agent, cf.\@ the example of Figure~\ref{fig:grid}. Regardless of the
quality of the chosen actions all state-values need to iteratively increase from the starting state
(usually \(0\)) to the base level imposed by the average reward \(\avgrew^{\pol}(s)/(1-\gamma)\)
offset by \(V^{\pol}(s) + e_{\gamma}^{\pol}(s)\). As usually there are more actions which are
sub-optimal in comparison to the number of optimal ones the agent will more likely choose such an
action in usually applied tabula rasa learning. Thus there is a high chance that cycles form.


%
%

\begin{example}

  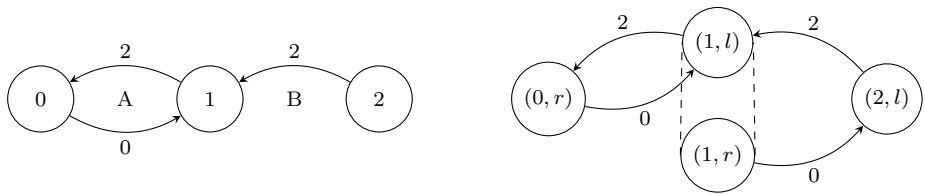
\begin{figure}[t!]
    \centering
    \begin{tikzpicture}[thin, scale=0.75]
  \draw (-3,0) node(0) [circle,draw,minimum size=25] {\footnotesize \(0\)};
  \draw (0,0)  node(1) [circle,draw,minimum size=25] {\footnotesize \(1\)};
  \draw (3,0)  node(2) [circle,draw,minimum size=25] {\footnotesize \(2\)};

  \path[thin, ->, bend right, >=stealth] (1) edge[above] node {\footnotesize\(2\)} (0);
  \path[thin, ->, bend right, >=stealth] (2) edge[above] node {\footnotesize\(2\)} (1);
  \path[thin, ->, bend right, >=stealth] (0) edge[below] node {\footnotesize\(0\)} (1);

  \draw (-1.5,0) node[] { A };
  \draw (1.5,0) node[] { B };

  \draw (6,0) node(B0) [circle,draw,minimum size=25] {\footnotesize \((0,r)\)};
  \draw (9,1.0)  node(B1l) [circle,draw,minimum size=25] {\footnotesize \((1,l)\)};
  \draw (9,-1.0)  node(B1r) [circle,draw,minimum size=25] {\footnotesize \((1,r)\)};
  \draw (12,0)  node(B2) [circle,draw,minimum size=25] {\footnotesize \((2,l)\)};

  \path[thin, ->, bend right, >=stealth] (B1l) edge[above] node {\footnotesize\(2\)} (B0);
  \path[thin, ->, bend right, >=stealth] (B2) edge[above] node {\footnotesize\(2\)} (B1l);
  \path[thin, ->, bend right, >=stealth] (B0) edge[below] node {\footnotesize\(0\)} (B1l);
  \path[thin, ->, bend right, >=stealth] (B1r) edge[below] node {\footnotesize\(0\)} (B2);

  \path[thin, dashed] (B1r.west) edge[below] node {} (B1l.west);
  \path[thin, dashed] (B1r.east) edge[below] node {} (B1l.east);

\end{tikzpicture}

    \caption{\label{fig:three-states2} The Blackwell-optimal policy \(\pol_{A}\) of the MDP of
      Figure~\ref{fig:three-states} as model-based (left) and model-free (right) version}
  \end{figure}

  Reconsider the task depicted in Figure~\ref{fig:three-states}. The (Blackwell-)optimal MDP
  \(\pol_{A}\), that is, the one that chooses the A-loop, is shown on the left side of
  Figure~\ref{fig:three-states2}.
  %
  %
  The right side of Figure~\ref{fig:three-states2} shows the same MDP for the model-free version. In
  model-free reinforcement learning state-action pairs as opposed to state values are estimated.
  %
  %
  %
  The dashed line indicates that these states are connected and thus the agent has to choose among
  them. In this case the bias values are \(V^{\pol_{A}}((0,r)) = -0.5\),
  \(V^{\pol_{A}}((1,l))=V^{\pol_{A}}((1,r))=0.5\), and \(V^{\pol_{A}}((2,l))=1.5\).
  %
  %

  When using standard discounted reinforcement learning, that is, estimating $V^{\pol}_{\gamma}(s)$
  the average reward of $1$ scales the state-action values. E.g.\@ for a discount factor of
  $\gamma=0.99$ the inferred values for a converged system are
  \(V_{0.99}^{\pol_{A}}((0,r)) = 99.497\), \(V_{0.99}^{\pol_{A}}((1,l))= 100.502\),
  \(V_{0.99}^{\pol_{A}}((1,r))=100.482\), and \(V_{0.99}^{\pol_{A}}((2,l))=101.497\). For all
  states the discounted state-value consists of the scaled average reward \(1/0.01 = 100\) and the
  bias value plus the error term.
  The duration of the process of learning the state values is unintentionally increased as the
  scaled average reward has to be learned in an iterative manner.
  Furthermore, the values itself are hard to interpret and the marginal difference
  between optimal and non-optimal action as compared to their actual values increase the likelihood
  of choosing sub-optimal actions, especially when function approximation is used to represent
  \(V_\gamma(s)\). In this example the difference of the state values for choosing action \(l\) over
  action \(r\) in state \(1\) is given by \(0.02\) as compared to their mean state value of
  \(100.492\).
\end{example}


\section{Average Reward Adjusted Discounted Reinforcement Learning}%
\label{sec:Average_Reward_Adjusted_Discounted_Reinforcement_Learning}

This section establishes the average reward adjusted discounted reinforcement learning algorithm.
Furthermore, we provide the Bellman Equations and an optimality discussion of the presented
algorithm.
For the rest of the paper we assume unichain MDPs, that is, we restrict our method to MDPs that
posses a scalar average reward value \(\avgrew^{\pol}\). In case of multichain MDPs we refer to the
companion paper \textit{Near-Blackwell-Optimal Average Reward Reinforcement
  Learning}\footnote{Working paper.}.

Starting from the Bellman Equations we derive the average reward adjusted reinforcement learning
equation using the Laurent series expansion provided in Equation~\ref{eq:laurent}. However, let us
first introduce a notion for the state values which are adjusted of the average reward.

\begin{definition}
  We define the \textit{average reward adjusted discounted state value} $\X_{\gamma}^{\pol}(s)$ of a state $s$
  under policy $\pol$ and with discount factor $0 \leqslant \gamma \leqslant 1$ as 
  \begin{align*}
    \X_{\gamma}^{\pol}(s) \defsym V^{\pol}(s) + e_{\gamma}^{\pol}(s)\tpkt
  \end{align*}
  This can be reformulated to
  \(\X_{\gamma}^{\pol}(s) = V_{\gamma}^{\pol}(s) - \frac{\avgrew^{\pol}}{1-\gamma} = \lim_{N \to
    \infty} E[\sum_{t=0}^{N-1} \gamma^{t} R_{t}^{\pol}(s)] - \frac{\avgrew^{\pol}}{1-\gamma}\), thus
  our definition is a reformulation of the average-adjusted reward values of
  Schwartz~\cite{schwartz1993reinforcement,schwartz1993thinking}.
\end{definition}

A major problem occurring at average reward RL is that the bias values are not uniquely defined
without solving the first set of constraints defined by the error term addends
(see~\cite[p.346]{Puterman94,Mahadevan96_SensitiveDiscountOptimalityUnifyingDiscountedAndAverageRewardReinforcementLearning}).
We could overcome this issue by simply requiring \(\gamma\) to be strictly less than \(1\).
However, actually our algorithm does not require the exact solution for \(V^{\pol}(s)\), but a
solution which is offset suffices. Clearly this observation reduces the required iteration steps
tremendously as finding the exact solution, especially for large discount factors, is tedious.
Therefore, we allow to set \(\gamma = 1\), which induces \(\X_{\gamma}^{\pol}(s) = V^{\pol}(s)+u\),
where \(u\) is for unichain MDPs a scalar value independent of \(s\), i.e.\@ equivalent for all
states of the MDP~\cite[p.346]{Puterman94}.
%
%
If we are interested in correct bias values, i.e. \(\gamma\) is close but strictly less than \(1\),
our approach is a tremendous advantage over average reward RL as it reduces the number of iterative
learning steps by requiring only a single constraint per state plus one for the scalar average
reward value. That is, for an MDP with \(N\) states only one more constraint (\(N+1)\) has to be
solved in \ARA{} as compared to (at least) \(2N+1\) nested constraints for average reward RL.
Therefore, it is cheap to compute \(\X_{\gamma}^{\pol}(s)\), while it is rather expensive to find
the correct values of \(V^{\pol}(s)\) directly, especially in an iterative manner as RL is.


\subsection{Bellman Equations}

Using the above notion we are able to derive the average reward adjusted discount reinforcement
learning state balance equation. To do so we make use of an equivalence found in the derivation of
the Bellman equation for \(V_{\gamma}^{\pol}(s)\) (see e.g.~\cite[p.70]{sutton1998introduction}) and
transform it to our needs. As in~\cite[p.66]{sutton1998introduction} we use the notation
\(R_{s,s'}^{a} = \E[r_{t} \mid s_{t}=s, a_{t} = a, s_{t+1} = s']\) for the expected reward to
receive when traversing from any current state \(s\) to state \(s'\) using action \(a\) and
\(\pol(a \mid s)\) the probability of taking action \(a\) in state \(s\) as given by policy \(\pol\).
\begin{align*}
  V_{\gamma}^{\pol}(s) & = \E_{\pol}[r_{t} + \gamma V_{\gamma}^{\pol}(s_{t+1})   \mid   s_{t} = s] \\
  \frac{\avgrew^{\pol}}{1-\gamma} + V^{\pol}(s) + e_{\gamma}^{\pol}(s) & = \E_{\pol}[r_{t} + \gamma ( \frac{\avgrew^{\pol}}{1-\gamma} + V^{\pol}(s_{t+1}) + e_{\gamma}^{\pol}(s_{t+1}))   \mid   s_{t} = s]\\
  \X_{\gamma}^{\pol}(s) & = \E_{\pol}[r_{t} + \gamma \X_{\gamma}^{\pol}(s_{t+1}) - \avgrew^{\pol} \mid   s_{t} = s]\\
  \X_{\gamma}^{\pol}(s) & = \sum_{a}\pol(a \mid s) \sum_{s'} p(s' \mid s_{t}=s, a_{t}=a) [R_{s,s'}^{a} + \gamma \X_{\gamma}^{\pol}(s') - \avgrew^{\pol}]
\end{align*}
Thus, we can compute the average reward adjusted discounted state value \(\X_{\gamma}^{\pol}(s)\) of
a state \(s\) by the returned reward, the adjusted discounted state value
\(\X_{\gamma}^{\pol}(s_{t+1})\) of the next state \(s_{t+1}\) and the average reward
\(\avgrew^{\pol}\). This is very similar to the Bellman equation of standard discounted RL, cf.\@
the first line.
Further, note that in line three we use the equivalence
$\avgrew^{\pol}(s) = \E_{\pol}[\avgrew^{\pol}(s)]$ described by the first addend of the Laurent
series expansion (see~\cite{MillerVeinott1969}).


In the same manner we derive the Bellman optimality equation for average reward adjusted discounted
reinforcement learning, cf.~\cite[p.76]{sutton1998introduction}.
\begin{align*}
  V_{\gamma}^{\polopt}(s) & = \max_{a} \E_{\polopt}[r_{t} + \gamma V_{\gamma}^{\polopt}(s_{t+1})   \mid   s_{t} = s] \\
  \frac{\avgrew^{\polopt}}{1-\gamma} + V^{\polopt}(s) + e_{\gamma}^{\polopt}(s) & = \max_{a} \E_{\polopt}[r_{t} + \gamma ( \frac{\avgrew^{\polopt}}{1-\gamma} + V^{\polopt}(s) + e_{\gamma}^{\polopt}(s))   \mid   s_{t} = s]\\
  \X_{\gamma}^{\polopt}(s) & = \max_{a} \E_{\polopt}[r_{t} + \gamma \X_{\gamma}^{\polopt}(s) - \avgrew^{\polopt}  \mid   s_{t} = s]\\
  \X_{\gamma}^{\polopt}(s) & = \max_{a} \sum_{s'}p(s' \mid s_{t}=s, a_{t}=a) [R_{s,s'}^{a} + \gamma \X_{\gamma}^{\polopt}(s') - \avgrew^{\polopt}]
\end{align*}
Again like in standard discounted RL, also in \ARA{} the value for the optimal policy equals the
expected value of the best action from that state~\cite[p.76]{sutton1998introduction}, where the
average reward is subtracted accordingly. This is very pleasing as we can use the same ideas in
terms of exploration/exploitation and state value acquisition as known from standard discounted
RL\@. 

\subsection{Near-Blackwell-Optimal Algorithm}
\label{subsec:Algorithm}

The reinforcement learning algorithm is depicted in Algorithm~\ref{alg:near}. In model-free methods,
which is what we aim for, state-action pairs are computed instead of state values only. Therefore,
the algorithm operates on state-action tuples, where for simplicity we write
\(\X^\pol_{\gamma}(s,a)\) instead of \(\X^\pol_{\gamma}((s,a))\) and the Bellman optimality equation
is adopted as in~\cite[p.76]{sutton1998introduction}, s.t.\@ the agent is able to selectively choose
actions among multiple states, cf. Figure~\ref{fig:three-states2}.
After initialising all values the agent enters the loop in which the first task is to choose an
action (step 3). In this action selection process we utilise an \(\epsilon\)-sensitive lexicographic
order \(\lexeq\) defined as \((a_{1},\ldots,a_{n}) = a \lexeq b = (b_{1},\ldots,b_{n})\) if and
only if \(| a_{j} - b_{j} | \leqslant \epsilon\) for all \(j < i\) and
\(|a_{i} - b_{i}| > \epsilon\). Note that the resulting sets of actions may not be disjoint.
Although this is an unusual order in programming, finding the set of maximizing values as in our
algorithm is straightforward and thus cheap to compute. In case the resulting set of actions
contains more than one element a random action of this set of actions is chosen.

\begin{algorithm}[t!]
  \begin{algorithmic}[1]
    \State{}Initialize state \(s_{0}$, $\avgrew^{\pol} = 0\), \(\X_{\cdot}^{\pol}(\cdot, \cdot) = 0\), set an
    exploration rate \(0 \leqslant
    p_{exp} \leqslant 1\), exponential smoothing learning rates \(0 < \alpha, \gamma < 1\), and
    discount factors \(0.5 \leqslant \gamma_{0} < \gamma_{1} \leqslant 1\), where \(\gamma_{1} = 1\) is
    usually a good choice.
    \While{the stopping criterion is not fulfilled}
    \State{}\begin{minipage}[t]{0.9\textwidth} With probability \(p_{exp}\) choose a random action
      and probability \(1-p_{exp}\) one that fulfills \newline
      \(\max_{a}\lexeq(\X^{\pol}_{\gamma_{1}}(s_{t},a),\X^\pol_ {\gamma_{0}}(s_{t},a))\).
    \end{minipage}
    \State{}Carry out action \(a_{t}\), observe reward \(r_{t}\) and resulting state \(s_{t+1}\).
    \If{a non-random action was chosen}
    \begin{align*}
      \avgrew^{\pol}  \gets (1- \alpha) \avgrew^{\pol} + \alpha [r_{t} + \max_{a}\X^\pol_{\gamma_1}(s_{t+1},a) - \X^\pol_{\gamma_1}(s_{t},a_{t})]
    \end{align*}
    \EndIf
    \State{}Update the average reward adjusted discounted state-values.
    \begin{align*}
      \X^\pol_{\gamma_0}(s_{t},a_{t}) & \gets (1-\gamma) \X^\pol_{\gamma_0}(s_{t},a_{t}) + \gamma [r_{t} + \gamma_{0} \max_{a} \X^\pol_{\gamma_0}(s_{t+1},a) - \avgrew^{\pol}] \\
      \X^\pol_{\gamma_1}(s_{t},a_{t}) & \gets (1-\gamma) \X^\pol_{\gamma_1}(s_{t},a_{t}) + \gamma [r_{t} + \gamma_{1} \max_{a} \X^\pol_{\gamma_1}(s_{t+1},a) - \avgrew^{\pol}]
    \end{align*}
    \State{} Set \(s \gets s'\), \(t \gets t+1\) and decay parameters
    \EndWhile{}
  \end{algorithmic}
  \caption{\label{alg:near}Model-free tabular near-Blackwell-optimal RL algorithm for unichain MDPs}
\end{algorithm}

As usual in reinforcement learning the state-values are exponentially
smoothed using parameters \(\alpha\) and \(\gamma\). In step 4 the
chosen action is carried out, while in step 5 the average reward
estimate is calculated.
%
As we aim for an estimate of \(\avgrew^{\polopt}\) we only update the value in case the
greedy action was chosen.
%
The formula used is a reformulation of the second addend of the Laurent series
expansion~(see~\cite{MillerVeinott1969} or~\cite[p.346]{Puterman94} for details):
\(\avgrew^{\pol}(s) + V^{\pol}(s) - E[V^{\pol}(s)] = R_{t}(s)\).
Like~\cite{Tadepalli98_ModelbasedAverageRewardReinforcementLearning} we also observed that the
average reward has to be updated by this formula and not by exponentially smoothing the actual
observed rewards as done in the algorithms of
Mahadevan~\cite{Mahadevan96_AnAveragerewardReinforcementLearningAlgorithmForComputingBiasoptimalPolicies,Mahadevan96_SensitiveDiscountOptimalityUnifyingDiscountedAndAverageRewardReinforcementLearning},
which likely leads to sub-optimal policies. Furthermore, we adapted the above given algorithm by
adding an exponentially smoothed\footnote{With rate \(\frac{1}{50}\) and update of \(97.5\%\) of the
  current reward in every period.} bound from below for the average reward value. The idea is that,
once a policy was established for some time, it does not make sense to aim for policies with smaller
average reward.

Finally the state-values are updated according to the average reward adjusted Bellman optimality
equation derived above (step 6) and the environment is updated to the next state and time period
(step 7).
In the next subsection we will introduce further optimality criteria of RL and prove that for a
given MDP the discount factors \(\gamma_{0}\) and \(\gamma_{1}\) and the comparison measure
\(\epsilon\) can be chosen s.t.\@ under the assumption of correctly approximated values the
algorithm produces Blackwell-optimal policies.



\subsection{Optimality Criteria}
\label{subsec:Optimality_Criteria}

In terms of optimality, we consider the notion of \(n\)-discount optimality as it is the broadest
approach of optimality criteria in reinforcement learning, and further for a sufficiently large
\(n\) it is known that there always exists a policy which is optimal~\cite{Blackwell62,Veinott69}. For a
comprehensive discussion of optimality criteria in reinforcement learning we refer
to~\cite{Mahadevan96_OptimalityCriteriaInReinforcementLearning}.


\begin{definition}%
  \label{def:n-discount-optimality}
  Due to Veinott~\cite{Veinott69} for MDPs a policy \(\polopt\) is \emph{\(n\)-discount-optimal} for
  \(n=-1,0,1,\ldots\) for all states \(s \in \States\) with discount factor
  \(\gamma\) 
  if and only if
  \begin{align*}
    \lim_{\gamma \to 1}(1-\gamma)^{-n}\ (V_{\gamma}^{\polopt}(s) - V_{\gamma}^{\pol}(s)) \geqslant 0 \tpkt
  \end{align*}
\end{definition}

As a policy can only be \(m\)-discount optimal if it is \(n\)-discount-optimal for all
\(n < m\)~\cite{Puterman94,Veinott69}, this leads to the component-wise comparison when greedily
choosing actions, cf.\@ the action selection process of the algorithm.

\begin{definition}
  If a policy is \(\infty\)-discount-optimal then it is said to be
  \textit{Blackwell-optimal}~\cite{Blackwell62}.
\end{definition}

That is, Blackwell-optimal policies are the in the sense of \(n\)-discount-optimality the best
achievable policies that first optimise for the highest gain, as we have for \(n=-1\) a measure for
gain-optimality~\cite{Mahadevan96_AverageRewardReinforcementLearningFoundationsAlgorithmsAndEmpiricalResults},
then for \(n=0\) for
bias-optimality~\cite{Mahadevan96_AverageRewardReinforcementLearningFoundationsAlgorithmsAndEmpiricalResults},
and as we will see for \(n \geqslant 1\) it maximises for the greatest error term. For an agent that
either expects to have infinitely many time to collect rewards, or one that is unaware when the
system will halt, this is the most sensible approach.

There are two known possibilities to incorporate the expected reward to be collected in the future.
The first one is to use a single discounted value, while the other approach in general incorporates
solving of infinitely many constraints. The following definition separates these kinds of
algorithms.
\begin{definition}
  If an algorithm infers for any MDP bias-optimal policies and for a given MDP can in theory be
  configured to infer \(\infty\)-discount-optimal policies, but in practise this ability is
  naturally limited due to the finite accuracy of floating-point representation of modern computer
  systems, it is said to be \textit{near-Blackwell-optimal} under the given computer system. An
  according to a near-Blackwell-optimal algorithm inferred Blackwell-optimal policy is called
  near-Blackwell-optimal.
\end{definition}
This definition is of practical relevance, as it defines a group of algorithms that are by far less
computationally expensive in comparison to ones that solve infinitely many constraints, but are able
to deduce sufficiently optimal policies. To the best of our knowledge there is no Blackwell-optimal
algorithm that neither requires infinitely many constraints to be solved nor is restricted by the
floating point precision.

\subsection{Optimality Analysis}
\label{subsec:Optimality_Analysis}

In this section we will analyse the procedure of Algorithm~\ref{alg:near} by the above definitions of
\(n\)-discount-optimality, cf.\@ Definition~\ref{def:n-discount-optimality}. Thus, we will start with
\(n=-1\) and then proceed from there.

\subsubsection{\(\mathit{(-1)}\)-Discount-Optimality}
%
For the case of \(n=-1\) this leads to
gain-optimality~\cite{Mahadevan96_AverageRewardReinforcementLearningFoundationsAlgorithmsAndEmpiricalResults},
defined as \(\avgrew^{\polopt}(s) - \avgrew^{\pol}(s) \geqslant 0\) for all policies \(\pol\) and
states \(s \in \States\). This means that a \((-1)\)-discount-optimal agent puts its highest
priority to maximising for the greatest average reward. However, recall that in unichain MDPs the
average rewards is equal among all states. Thus, all bias values are assessed with the same average
reward, i.e.\@ the agent automatically maximises for the greatest gain.
Note that by definition, in case of a possibly falsely predetermined and fixed average reward value,
the bias values are estimated according to the given policy induced by the average reward.
Furthermore, if the average reward is fixed to a wrong value the bias values are similarly shifted
as in the standard discounted framework. Therefore, we highly recommend to infer the average reward
automatically as specified in the algorithm.

\begin{example}
  Reconsider the MDP of Figure~\ref{fig:printer} and the discount factor \(\gamma = 0.99\). If we
  fix the average reward to \(\avgrew^{\pol}_{\mathsf{fix}} = 1\) as opposed to the correct value of
  \(\avgrew^{\pol} = 2\), the algorithm infers values
  \(\X_{\mathsf{fix},\gamma}^{\pol}(s) = \frac{\avgrew^{\pol} -
    \avgrew_{\mathsf{fix}}^{\pol}}{1-\gamma} + \X_{\gamma}^{\pol}(s) = 100 + \X_{\gamma}^{\pol}(s)\)
  instead.
\end{example}

\subsubsection{\(\mathit{0}\)-Discount-Optimality}
%
In the case of \(n=0\) the criteria of \(n\)-discount-optimality describes bias-optimality with
\(V^{\polopt}(s) - V^{\pol}(s) \geqslant 0\) for all policies \(\pol\) and states
\(s \in
\States\)~\cite{Mahadevan96_AverageRewardReinforcementLearningFoundationsAlgorithmsAndEmpiricalResults}.
Thus for the algorithm the first decision level is to maximise for the policy yielding the highest
bias values. As we have \(\lim_{\gamma \to 1} e_{\gamma}^{\pol}(s) = 0\), we know that according to
the chosen \(\epsilon\) for sufficiently large \(\gamma_{1}\) the agent selects the set of actions
that maximise \(V^{\pol}(s)\).

\begin{theorem}
  For a sufficiently large \(\gamma_{1}< 1\), where sufficiently large means that for all states
  \(s\) we have \(|e_{\gamma_{1}}^{\pol}(s)| \leqslant \epsilon\), i.e.\@ it depends on the
  parameter \(\epsilon\), a \(0\)-discount-optimal agent chooses an action among the set of possible
  actions that maximise \(\X_{\gamma_{1}}^{\pol}(s)\).
\end{theorem}

\begin{proof}
  Recall that \(\lim_{\gamma \to 1} e_{\gamma}^{\pol}(s) = 0\) and
  \(\avgrew^{\polopt} = \avgrew^{\pol}\). We have
  \begin{align*}
    \lim_{\gamma \to 1} (1-\gamma)^{0}  (V_{\gamma}^{\polopt}(s) - V_{\gamma}^{\pol}(s)) \geqslant 0\\
    \lim_{\gamma \to 1} (\frac{\avgrew^{\polopt} - \avgrew^{\pol}}{1-\gamma} + V^{\polopt}(s) - V^{\pol}(s) + e_{\gamma}^{\polopt}(s) - e_{\gamma}^{\pol}(s)) \geqslant 0\\
    V^{\polopt}(s) - V^{\pol}(s) \geqslant 0
  \end{align*}
  meaning that a \(0\)-discount-optimal policy \(\pol\) has to maximise the bias values
  \(V^{\pol}(s)\) for all states \(s\).
  By definition of the \(\epsilon\)-sensitive lexicographic order (\(a \lexeq b\) if and only if
  \(| a_{j} - b_{j} | \leqslant \epsilon\)), we have for a sufficiently large \(\gamma_{1}\)-value
  \(|e_{\gamma_{1}}^{\pol}(s)| \leqslant \epsilon\) and thus
  \(|V^{\pol}(s) - \X_{\gamma_{1}}^{\pol}(s)| \leqslant \epsilon \) for all states \(s\). Thus the claim
  follows.\qed{}
\end{proof}

\subsubsection{\(\infty\)-Discount-Optimality}
%

Finally, in case $n \geqslant 1$, the agent has to choose actions that satisfy
\(e_{\gamma}^{\polopt}(s) - e_{\gamma}^{\pol}(s) \geqslant 0\) once \(\gamma \to 1\). That is, the
agent must maximise the error term, which means for \(n \geqslant 1\) we analyse the case where
\(\gamma < 1\) is used to incorporate short term rewards into the discounted state value, i.e.\@ how
long-sighted the agent shall be.
Therefore, the number of actions to reach a desired goal or path is taken into account, as well as
when and how much rewards are collected.
%
However, as the error term depends on infinitely many sub-terms simply estimating these and summing
up does not work.
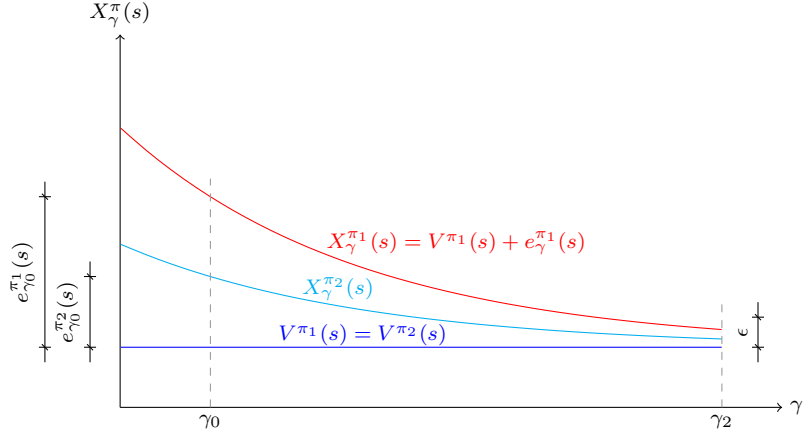
\begin{figure*}[t]
  \centering
  \begin{tikzpicture}[scale=8,
  declare function={
    red(\x) = {(0.08^((\x+0.4)))+0.2};
    cyan(\x) = {(0.08^((\x+0.7)))+0.2};
  },
  ]

  \draw[->] (0,0.1) -- (1.1,0.1) node[right] {$\gamma$};
  \draw[->] (0,0.1) -- (0,0.72) node[above] {$\X_{\gamma}^{\pol}(s)$};
  \draw[scale=1,domain=0.001:1.0,smooth,variable=\x,blue] plot ({\x},{0.2});
  \draw[scale=1,domain=0.001:1.0,smooth,variable=\x,red] plot ({\x},{red(\x)});
  \draw[scale=1,domain=0.001:1.0,smooth,variable=\x,cyan] plot ({\x},{cyan(\x)});

  \tikzmath{\lx = 0.25; \eps=0.04; \delY=0.015;};
  \tikzmath{\yr = red(\lx+2*\eps)+\delY;};
  \tikzmath{\yc = cyan(\lx+\eps)+\delY;};
  \draw[red]  (\lx+2*\eps,\yr)   node[anchor=west] () {$\X_\gamma^{\pol_{1}}(s) = V^{\pol_{1}}(s) + e_{\gamma}^{\pol_{1}}(s)$};
  \draw[cyan] (\lx+\eps,\yc)   node[anchor=west] () {$\X_\gamma^{\pol_{2}}(s)$};
  \draw[blue] (\lx,0.2+\delY+0.0055) node[anchor=west]     () {$V^{\pol_{1}}(s) = V^{\pol_{2}}(s)$};

  \tikzmath{\e = red(0.15);}; 
  \tikzmath{\xe = -0.075;};

  \draw[black, dashed, black!40] (0.15, 0.1) -- (0.15, \e+0.03);
  \draw (0.15, 0.1) node[below] {\(\gamma_{0}\)};

  \draw[black] (\xe-0.05, 0.2-0.025) -- (\xe-0.05, \e+0.025) node[midway,xshift=-0.3cm,rotate=90] {\footnotesize \(e_{\gamma_{0}}^{\pol_{1}}(s)\)};
  \draw[black] (\xe-0.06, 0.2)       -- (\xe-0.04, 0.2);
  \draw[black] (\xe-0.055, 0.205)    -- (\xe-0.045, 0.195);
  \draw[black] (\xe-0.06, \e)        -- (\xe-0.04, \e);
  \draw[black] (\xe-0.055, \e+0.005) -- (\xe-0.045, \e-0.005);

  \tikzmath{\e = cyan(0.15);};      
  \tikzmath{\xe = 0;};
  \draw[black] (\xe-0.05, 0.2-0.025) -- (\xe-0.05, \e+0.025) node[midway,xshift=-0.3cm,rotate=90] {\footnotesize \(e_{\gamma_{0}}^{\pol_{2}}(s)\)};
  \draw[black] (\xe-0.06, 0.2)       -- (\xe-0.04, 0.2);
  \draw[black] (\xe-0.055, 0.205)    -- (\xe-0.045, 0.195);
  \draw[black] (\xe-0.06, \e)        -- (\xe-0.04, \e);
  \draw[black] (\xe-0.055, \e+0.005) -- (\xe-0.045, \e-0.005);

  \tikzmath{\e = 0.25;};
  \tikzmath{\xe = 1.11;};

  \draw[black, dashed, black!40] (1.0, 0.1) -- (1.0, \e+0.025);
  \draw (1.0, 0.1) node[below] {\(\gamma_{2}\)};

  \draw[black] (\xe-0.05, 0.2-0.025) -- (\xe-0.05, \e+0.025) node[midway,xshift=-0.2cm,rotate=90] {\footnotesize \(\epsilon\)};
  \draw[black] (\xe-0.06, 0.2)       -- (\xe-0.04, 0.2);
  \draw[black] (\xe-0.055, 0.205)    -- (\xe-0.045, 0.195);
  \draw[black] (\xe-0.06, \e)        -- (\xe-0.04, \e);
  \draw[black] (\xe-0.055, \e+0.005) -- (\xe-0.045, \e-0.005);





\end{tikzpicture}

  \caption{Visualisation of the strictly monotonically decreasing error terms $e_{\gamma}^{\pol_{0}}(s)$
    and $e_{\gamma}^{\pol_{2}}(s)$ as $\gamma$ approaches 1}%
  \label{fig:err}
\end{figure*}
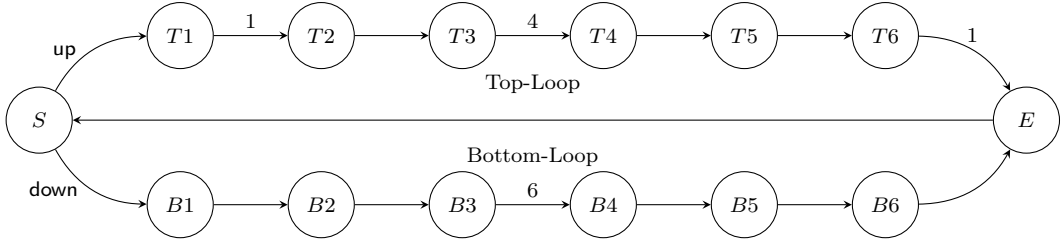
\begin{figure*}[t]
  \centering
  \begin{tikzpicture}[thin, scale=0.75]

  \draw (-2.5   ,0)  node(S)  [circle,draw,minimum size=25] {\footnotesize \(S\)};
  \draw (0.0  ,1.5)  node(T1) [circle,draw,minimum size=25] {\footnotesize \(T1\)};
  \draw (2.5  ,1.5)  node(T2) [circle,draw,minimum size=25] {\footnotesize \(T2\)};
  \draw (5.0  ,1.5)  node(T3) [circle,draw,minimum size=25] {\footnotesize \(T3\)};
  \draw (7.5  ,1.5)  node(T4) [circle,draw,minimum size=25] {\footnotesize \(T4\)};
  \draw (10.0 ,1.5)  node(T5) [circle,draw,minimum size=25] {\footnotesize \(T5\)};
  \draw (12.5 ,1.5)  node(T6) [circle,draw,minimum size=25] {\footnotesize \(T6\)};
  \draw (0.0  ,-1.5) node(B1) [circle,draw,minimum size=25] {\footnotesize \(B1\)};
  \draw (2.5  ,-1.5) node(B2) [circle,draw,minimum size=25] {\footnotesize \(B2\)};
  \draw (5.0  ,-1.5) node(B3) [circle,draw,minimum size=25] {\footnotesize \(B3\)};
  \draw (7.5  ,-1.5) node(B4) [circle,draw,minimum size=25] {\footnotesize \(B4\)};
  \draw (10.0 ,-1.5) node(B5) [circle,draw,minimum size=25] {\footnotesize \(B5\)};
  \draw (12.5 ,-1.5) node(B6) [circle,draw,minimum size=25] {\footnotesize \(B6\)};
  \draw (15.0 ,0)  node(E) [circle,draw,minimum size=25]    {\footnotesize \(E\)};

  \path[thin, ->, bend left, >=stealth] (S)  edge[above, near start] node[xshift=-3] {\footnotesize \textsf{up}} (T1);
  \path[thin, ->,            >=stealth] (T1) edge[above] node {\footnotesize \(1\) } (T2);
  \path[thin, ->,            >=stealth] (T2) edge[above] node {\footnotesize     } (T3);
  \path[thin, ->,            >=stealth] (T3) edge[above] node {\footnotesize\(4\)} (T4);
  \path[thin, ->,            >=stealth] (T4) edge[above] node {\footnotesize     } (T5);
  \path[thin, ->,            >=stealth] (T5) edge[above] node {\footnotesize     } (T6);
  \path[thin, ->, bend left, >=stealth] (T6) edge[above] node {\footnotesize\(1\)} (E);

  \path[thin, ->,bend right, >=stealth] (S)  edge[below, near start] node[xshift=-7] {\footnotesize \textsf{down} } (B1);
  \path[thin, ->,            >=stealth] (B1) edge[above] node {\footnotesize     } (B2);
  \path[thin, ->,            >=stealth] (B2) edge[above] node {\footnotesize     } (B3);
  \path[thin, ->,            >=stealth] (B3) edge[above] node {\footnotesize\(6\)} (B4);
  \path[thin, ->,            >=stealth] (B4) edge[above] node {\footnotesize     } (B5);
  \path[thin, ->,            >=stealth] (B5) edge[above] node {\footnotesize     } (B6);
  \path[thin, ->,bend right, >=stealth] (B6) edge[above] node {\footnotesize     } (E);

  \path[thin, ->,            >=stealth] (E)  edge[above] node   {\footnotesize} (S);

  \draw (6.25,0.65)  node[] { Top-Loop    };
  \draw (6.25,-0.65) node[] { Bottom-Loop };
\end{tikzpicture}

  \caption{An example MDPs for which the discount factor \(\gamma_{0}\) can be used to balance
    short- and long-sightedness for near-Blackwell-optimal algorithm }%
  \label{fig:parallel}
\end{figure*}
%
%
%
%
%
%

The RL algorithm depicted in Algorithm~\ref{alg:near} is unable to generally deduce
Blackwell-optimal policies. The cause is illustrated in Figure~\ref{fig:err}, where we assume the
error terms to be polynomials of the same degree. Let the policies \(\pol_{1}\), \(\pol_{2}\) choose
the actions \(\aOne{}\) and \(\aTwo{}\) resp.\@ of the set of \(0\)-discount-optimal stationary
actions in state \(s\) and then follow policy \(\pol\).
Note the different slopes of the two approximated values of \(\X_{\gamma}^{\cdot}(s)\). Therefore,
by interpolation we know that for very high values of \(\gamma\) we have
\(\X_{\gamma}^{\pol_{2}}(s) > \X_{\gamma}^{\pol_{1}}(s)\). However, the agent will choose action
\(\aOne{}\) as it maximises \(\X_{\gamma_{0}}^{\pol}(s)\). Therefore, this means that the parameter
\(\gamma_{0}\) defines how long-sighted the agent is. The MDP of Figure~\ref{fig:parallel} explains
the idea. The only action choice is in state \(S\), where the agent either decides to do the top- or
bottom-loop. For both loops the same amount of rewards are collected, thus the average reward
(\(0.75\)) and bias values (\(0.25\) for action \textsf{up}, \(0.375\) for action \textsf{down}) are
equal, regardless of the chosen policy. But only going \textsf{down} is Blackwell-optimal, as the
full amount of rewards is collected sooner. This also manifests in a higher bias value. However,
recall that the agent is unable to separate the actions by the bias values in case the same amount
of rewards are collected.
If we set \(\gamma_{0} = 0.50\) the agent deduces state-values
\(V_{\gamma_{0}}^{\pol_{\mathsf{Top}}}(S)=-0.480\) and
\(V_{\gamma_{0}}^{\pol_{\mathsf{Bottom}}}(S)=-0.746\) and thus like for any other value
\(\gamma_{0} < 0.84837\) chooses the top-loop.

This shows how for our algorithm \(\gamma_{0}\) functions exactly as the discount factor in standard
RL is supposed to do, due to the fact that the state values are adjusted: It can be used to balance
expected short-term and long-term rewards, without changing the main optimisation objective, i.e.\@
maximise for the highest average reward and bias values, before taking path lengths into account.
Especially for highly volatile systems, e.g.\@ stochastic production and control systems, being able
to set the long-sightedness can be an advantage over Blackwell-optimal agents. Nonetheless, setting
very high values for \(\gamma_{0}\) and using the \(\epsilon \approx 0\) for the comparison of
\(\X_{\gamma_{0}}^{\pol}(s)\) values, could be an approach in finding Blackwell-optimal policies for
many MDPs. But recall that for any arbitrary large \(\gamma_{0} < 1\) it is possible to construct
MDPs which lead to non-optimal policies~\cite{schwartz1993reinforcement}.

\begin{theorem}
  A \(n\)-discount-optimal agent for \(n \geqslant 1\) has to maximise the error term
  \(e_{\gamma}^{\pol}(s)\) once \(\gamma \to 1\).
\end{theorem}

\begin{proof}
  Recall that for unichain MDPs the average reward \(\avgrew^{\pol}(s)\) for all states \(s\) is
  equal and stated as \(\avgrew^{\pol}\). Furthermore, as the policy must be \(0\)-discount-optimal
  to be eligible for \(n\)-discount-optimality with \(n \geqslant 1\) we have
  \(V^{\polopt}(s) = V^{\pol}(s)\) for all states \(s\). Therefore,
  \begin{align*}
    \lim_{\gamma \to 1} (1-\gamma)^{-n} (V_{\gamma}^{\polopt}(s) - V_{\gamma}^{\pol}(s)) & \geqslant 0\\
    \lim_{\gamma \to 1} (1-\gamma)^{-n} ( \frac{\avgrew^{\polopt}(s)}{1-\gamma} + V^{\polopt}(s) + e_{\gamma}^{\polopt}(s)  
    - \frac{\avgrew^{\pol}(s)}{1-\gamma} - V^{\pol}(s) - e_{\gamma}^{\pol}(s) ) & \geqslant 0\\
    \lim_{\gamma \to 1} (1-\gamma)^{-n} ( e_{\gamma}^{\polopt}(s) - e_{\gamma}^{\pol}(s) ) & \geqslant 0\\
  \end{align*}
  which completes the proof as the error term does not depend on \(n\). However, note that as
  \(e_{\gamma}^{\polopt}(s)\) approaches 0 as \(\gamma \to 1\), we are interested in the cases
  where \(\gamma < 1\). That is, the error term incorporates the amount and number of steps until
  rewards are collected. \qed{}


\end{proof}

In other works, the error term is often split into its subterms by
\(e_{\gamma}^{\pol}(s) \defsym \sum_{m=1}^{\infty}(\frac{1-\gamma}{\gamma})^{m} \cdot y_{m}^{\pol}\)
(for the definition see~\cite{MillerVeinott1969}). 
Then for any \(n \geqslant 1\) the terms evaluate to maximising \(y_{n}^{\pol}\), as
for all subterms \(<n\) the values are equal due to \((n-1)\)-discount-optimality and for \(n>1\) the
terms evaluate to \(0\).
This leads to the approaches of average reward dynamic programming, where \(n\)-nested sets of
constraints are solved~\cite[e.g. p.511ff]{Puterman94}. Clearly, this straightforward approach is
computationally very expensive and infinite. Therefore, and also due to the imposed infinite
polynomial structure of the error term formula we refrain on adapting this strategy and rather let
the user choose an appropriate \(\gamma_{0}\) value for the provided situation.

Thus, for a given MDP and under the assumption of correct approximations and wisely selected
discount-factors \(\gamma_{0}\) and \(\gamma_{1}\) in combination of the chosen \(\epsilon\)-value
our algorithm is able to infer Blackwell-optimality policies. Nonetheless due to the accuracy of
floating-point representation of modern computer systems the previous statement is naturally
bounded. Therefore, the in this paper established reinforcement learning algorithm \ARA{} is
\textbf{near-Blackwell-optimal}.


\section{Experimental Evaluation}%
\label{sec:Experimental_Evaluation}

In this section we prove the viability of the algorithm with three examples. We compare the
algorithm to standard discounted RL, where we choose the widely applied
Q-Learning~\cite{sutton1998introduction,Watkins89_LearningFromDelayedRewards} technique as
appropriate model-free comparison method. The Q-Learning algorithm is given in
Algorithm~\ref{alg:qlearn}. We have adapted the parameter names accordingly to match the \ARA{}
algorithm from above.

\begin{algorithm}[t!]
  \begin{algorithmic}[1]
    \State{}Initialize state $s_{0}$, $Q_{\gamma_{1}}^{\pol}(\cdot, \cdot) = 0$, set an
    exploration rate \(0 \leqslant
    p_{exp} \leqslant 1\) and \(0 < \gamma, \gamma_{1} < 1\).
    \While{the stopping criterion is not fulfilled}
    \State{}\begin{minipage}[t]{0.9\textwidth} With probability \(p_{exp}\) choose a random action
      and probability \(1-p_{exp}\) one that fulfills \(\max_{a}Q^{\pol}_{\gamma_{1}}(s_{t},a)\) at the
      current state \(s_{t}\).
    \end{minipage}
    \State{}Carry out action \(a_{t}\), observe reward \(r_{t}\) and resulting state \(s_{t+1}\).
    \State{}Update the discounted state-values.
    \begin{align*}
      Q^\pol_{\gamma_1}(s_{t},a_{t}) & \gets (1-\gamma) Q^\pol_{\gamma_1}(s_{t},a_{t}) + \gamma [r_{t} + \gamma_{1} \max_{a'} Q^\pol_{\gamma_1}(s_{t+1},a')]
    \end{align*}
    \State{} Set \(s \gets s'\), \(t \gets t+1\) and decay parameters
    \EndWhile{}
  \end{algorithmic}
  \caption{\label{alg:qlearn}Watkins Q-Learning algorithm~\cite{Watkins89_LearningFromDelayedRewards}.
    Adapted from the version by Sutton~\cite[p.149]{sutton1998introduction}.}
\end{algorithm}

\subsection{Printer-Mail}%
\label{subsec:Printer-Mail}

Reconsider Figure~\ref{fig:printer} discussed above. The agent chooses either the printer loop or
the mail loop, whereas the mail loop returns a reward of \(20\) every tenth step and the printer
loop \(5\) every fifth step. As there is no stochastic in the reward function, nor the transition
function, the problem can be easily solved until convergence. To do so we set the learning rate
\(\gamma = 0.01\) and for \ARA{} \(\alpha =0.01\) exponentially decayed with rate \(0.25\) in
\(100k\) steps, where the minimum is set to \(10^{-6}\), and the discount factors
\(\gamma_{0} = 0.8, \gamma_{1} = 0.99\). For Q-Learning we repeated the experiment with discount
factors \(\gamma_{1} = 0.99\), \(\gamma_{1}=0.8\) and \(\gamma_{1}=0.50\).

The results are depicted in Figure~\ref{tbl:printer}, where the state-action values for
\((1, \mathsf{left})\) and \((1, \mathsf{right})\) are shown. The table reports the values of
\(\X_{\gamma_{1}}^{\pol}(\cdot)\) for \ARA{}, and \(Q_{\gamma_{1}}^{\pol}(\cdot)\) in case of
Q-Learning. The reported number of steps are measured until convergence, which we defined as no
state-value change in \(100k\) steps.
As \ARA{}, as well as Q-Learning with \(\gamma_{1}=0.99\), report a greater value for going
\textsf{right} than taking action \textsf{left}, both infer the optimal policy, while Q-Learning
with \(\gamma_{1} = 0.80\) and \(\gamma_{1}=0.50\) are unable to find it.
This shows clearly that it is crucial in standard discounted RL to use a discount factor close to
\(1\).
When setting \(\gamma_{1}=0.8027\) the discounted values are
\(V_{\gamma_{1}}^{\pol}(1, \mathsf{left}) = V_{\gamma_{1}}^{\pol}(1, \mathsf{right}) = 3.113\).
The average reward learned by \ARA{} is \(1.999\) and thus very close to the actual value of \(2\).
Note that as all states are assessed with the same average reward it is not necessarily to estimate
it perfectly, and thus small deviations have no impact on the policy of \ARA{}.
The reported steps to convergence for Q-Learning increase exponentially when increasing the discount
factor, which is due to the requirement of very high discount factors a rather unsatisfactory
behaviour. Especially as \ARA{} converges in \(10^{6}\) steps, while Q-Learning \(\gamma_{1}=0.99\)
requires a more than ten times longer learning phase.

\begin{table}
  \centering
  \ratab{}
  \begin{tabular}{@{}lrrrr@{}}
    \toprule
      & &  \multicolumn{3}{c}{Q-Learning} \\
    \cmidrule{3-5}
    Measure & \ARA{} & \(\gamma_{1}=0.99\) & \(\gamma_{1}=0.80\) & \(\gamma_{1}=0.50\)\\
\midrule
State \((1, \mathsf{left})\) & \(-13.349\) & \(186.509\) & \(3.046\) & 0.323\\
State \((1, \mathsf{right})\) & \(-8.787\) & \(191.070\) & \(3.011\) & 0.039\\
Steps in \(10^{6}\) & \(1.0\) & \(10.3\) & \(0.5\) & \(0.4\)\\
\bottomrule
  \end{tabular}
  \caption{\label{tbl:printer} The state-values for state \(1\) of the printer-mail MDP and the
    number of steps until convergence, where \ARA{} inferred an average reward of \(\avgrew^{\pol}=1.999\) }
\end{table}

\subsection{Gridworld}
\label{subsec:Gridworld}

We use a scaled up version of the MDP given in Figure~\ref{fig:grid} by increasing the state space
to a \(5 \times 5\) grid to make the optimisation task a little bit more challenging. The reward
function is adapted accordingly, where moving out of the grid is punished and otherwise a stochastic
reward of \(\Unif(0,8)\) is returned. The goal state is indifferent. Again, it is obvious that the
optimal policy is to traverse to \((0,0)\) as fast as possible. Further, the system is symmetric,
where states \((m,n)\) and \((n,m)\) are equal when also swapping the actions \textsf{left} with
\textsf{up} and \textsf{right} with \textsf{down} respectively. For the optimal policy the average
number to reach the goal state is \(4\) steps, with a reward of on average \(4\) and then there is
the random action which produces a reward of \(10\). Thus, the average reward for the optimal policy
is \(5.2\) and the average number of steps to reach the goal state should be \(5\). However, note
that we use no episodes, i.e.\@ there is no terminal state, as otherwise the Q-Learning algorithm
fails completely by ignoring the average reward in the terminal state, and thus producing policies
with the lowest estimate in the goal state.

We initialise the learning rates \(\alpha = \gamma = 0.01\), the discount factor
\(\gamma_{0} = 0.80\) and set \(\epsilon = 0.25\), \(p_{exp} = 1.00\). The learning rates and
exploration are exponentially decayed as follows. \(\alpha\) with a rate of \(0.50\) in \(50k\)
steps and a minimum of \(10^{-5}\), \(\gamma\) with rate of \(0.50\) in \(150k\) steps and
\(10^{-3}\), and finally the exploration with rate \(0.50\) in \(100k\) steps and a minimum of
\(0.01\). We execute \(500k\) learning steps before doing an evaluation run of \(10k\) steps for
which exploration and learning is disabled. The experiment, including the learning process, is
repeated \(40\) times with the same random number streams over the different setups.

\begin{table}
  \centering
  \ratab{}
  \begin{tabular}{@{}lrrrr@{}}
    \toprule
    & \multicolumn{2}{c}{Sum Reward}  & \multicolumn{2}{c}{Avg.\@ Steps to Goal}\\
    \cmidrule(r){2-3} \cmidrule(l){4-5}
 Algorithm & Mean & StdDev & Mean & StdDev\\
\midrule
\ARA{} \(\gamma_1=0.99\) & \cellcolor{gr}\(\textbf{51894.094}\) & \(234.260\) & \(\textbf{5.039}\) & \(0.047\)\\
\ARA{} \(\gamma_1=0.999\) & \cellcolor{gr}\(51878.069\) & \(282.335\) & \cellcolor{gr}\(5.063\) & \(0.054\)\\
    \ARA{} \(\gamma_1=1.00\) & \cellcolor{gr}\(51856.529\) & \(242.953\) & \cellcolor{gr}\(5.055\) & \(0.039\)\\
    \hhline{|~|-|~|-|~|}
Q-Learning  \(\gamma_1=0.99\) & \cellcolor{gr2}\(34409.464\) & \(3391.818\) & \cellcolor{gr2}\(7661.833\) & \(3518.656\)\\
Q-Learning  \(\gamma_1=0.999\) & \cellcolor{gr2}\(33931.917\) & \(3076.401\) & \cellcolor{gr2}\(7379.155\) & \(3745.914\)\\
Q-Learning  \(\gamma_1=0.50\) & \(30171.837\) & \(328.591\) & \(9999.000\) & \(0.000\)\\
\bottomrule
  \end{tabular}\\[0.5em]
  \caption{\label{tbl:grid} The results of the gridworld example, where all \ARA{} instances
  inferred an average reward of \(\avgrew^{\pol}=5.215\)}
\end{table}

Table~\ref{tbl:grid} presents the summary of the results for the gridworld experiment. We have
evaluated both methods with three different discount factors, namely \(0.99\), \(0.999\) and
\(1.00\) for \ARA{} and \(0.50\), \(0.99\) and \(0.999\) for Q-Learning.
We used the Friedman test with a significance level of \(p=0.05\) for a statistical analysis of the
mean sum of rewards and the mean average steps until the goal step. As expected the omnibus null
hypotheses (all samples are from the same distribution) are rejected (with \(3.223\text{e}{-33}\)
for the mean reward and \(1.736\text{e}{-34}\) for the mean avg.\@ steps to goal). Therefore, we
conducted pairwise Conover post-hoc tests adjusted by the Benjaminyi-Hochberg FDR
method~\cite{benjamini1995controlling} to reduce liberality (see the Appendix for the detailed
results).
Measures highlighted with the same shade of grey are not statistically distinguishable from each
other.
%
%
It can be seen that all \ARA{} instance perform very well, outperforming Q-Learning in terms of
amount of collected rewards, as well as finding the shortest path to the goal state. More precisely
all \ARA{} instances collect on average over all \(40\) experiments a reward of \(5.19\) per step,
which is only \(0.01\) less than the optimum of \(5.2\), while the best Q-Learning variant receives
\(3.44\) which means that Q-Learning not even learns to avoid to steer of the grid. The standard
deviations, especially of the average number of steps to reach the goal state, undermine the great
performance of \ARA{} as it shows how stable the algorithm works while the Q-Learning results are
rather unstable.
\ARA{} \(\gamma_1 = 0.99\) performs significantly better in terms of mean average number of steps to
the goal state than the other setups, while for the mean sum of reward all \ARA{} instances are
significantly indifferent. This is due to the very small deviations of the former measure.
Further the average steps show Q-Learning is unable to find even close-to near-optimal policies. The
estimated average reward by the \ARA{} algorithm is \(\avgrew^{\pol}=5.215\) for all setups, and
thus almost perfectly matches the analytically inferred \(5.20\). An explanation, why Q-Learning
performs that poorly can be seen in Figure~\ref{fig:grid4}. It provides discounted state-values as
estimated by the Q-Learning \(\gamma_{1}=0.99\) variant after \(1\) Million steps of learning.
Here we can see a cluster forming at the state-actions pairs of state \((0,4)\), which have with
almost \(400\) a far higher evaluation than the goal state \((0,0)\) of around \(100\). Clearly, the
surrounding states adapt accordingly worsening the situation.

\begin{figure}[tb]
  \centering
  \begin{tikzpicture}[thin, scale=1.5]
  \newcommand\maxX{4}         
  \newcommand\maxY{0}         
  \newcommand\goalX{0}         
  \newcommand\goalY{0}         

  \tikzmath{\maxXP = \maxX + 1; \maxYP = \maxY + 1;};

  \foreach \x in {0,...,\maxXP} {
    \foreach \y in {0,...,\maxYP} {
      \draw (\x,0) -- (\x,\maxYP) node[] {};
      \draw (0,\y) -- (\maxXP,\y) node[] {};
    }
    \draw (\x,0) -- (\x,-0.25) node[] {};
  }
  \foreach \x in {0,...,\maxX} {
    \draw (\x+0.5,-0.25) node[] (\x+0.5,-0.25) {\(\ldots\)};
  }

  \foreach \x in {0,...,\maxX} {
    \foreach \y in {0,...,\maxY} {
      \tikzmath{\xP = int(\maxX - \x); \yP = int(\maxY - \y);};
      \draw (\x+0.5, \y+0.5) node[] {(\yP,\x)};
    }
  }

  \draw (\goalX+0.5, \maxY-\goalY+0.3) node[] (\goalX,\goalY,rand) {\tiny 101.260};

  \tikzmath{\x = 1; \y = 0; \xP = int(\maxX - \x); \yP = int(\maxY - \y);};
  \draw (\x+0.5, \y+0.9) node[] (\yP,\x,up)                  {\tiny \( 23.113 \)};
  \draw (\x+0.5, \y+0.1) node[] (\yP,\x,down)                {\tiny \( 37.331 \)};
  \draw (\x+0.1, \y+0.5) node[rotate=90] (\yP,\x,left)       {\tiny \( 67.160 \)};
  \draw (\x+0.9, \y+0.5) node[rotate=90] (\yP,\x,right)      {\tiny \( 61.783 \)};

  \tikzmath{\x = 2; \y = 0; \xP = int(\maxX - \x); \yP = int(\maxY - \y);};
  \draw (\x+0.5, \y+0.9) node[] (\yP,\x,up)                  {\tiny \( 162.243 \)};
  \draw (\x+0.5, \y+0.1) node[] (\yP,\x,down)                {\tiny \( 153.653 \)};
  \draw (\x+0.1, \y+0.5) node[rotate=90] (\yP,\x,left)       {\tiny \( 49.812  \)};
  \draw (\x+0.9, \y+0.5) node[rotate=90] (\yP,\x,right)      {\tiny \( 349.515 \)};

  \tikzmath{\x = 3; \y = 0; \xP = int(\maxX - \x); \yP = int(\maxY - \y);};
  \draw (\x+0.5, \y+0.9) node[] (\yP,\x,up)                  {\tiny \( 354.839 \)};
  \draw (\x+0.5, \y+0.1) node[] (\yP,\x,down)                {\tiny \( 351.907 \)};
  \draw (\x+0.1, \y+0.5) node[rotate=90] (\yP,\x,left)       {\tiny \( 306.761 \)};
  \draw (\x+0.9, \y+0.5) node[rotate=90] (\yP,\x,right)      {\tiny \( 399.560 \)};

  \tikzmath{\x = 4; \y = 0; \xP = int(\maxX - \x); \yP = int(\maxY - \y);};
  \draw (\x+0.5, \y+0.9) node[] (\yP,\x,up)                  {\tiny \( 398.391 \)};
  \draw (\x+0.5, \y+0.1) node[] (\yP,\x,down)                {\tiny \( 399.539 \)};
  \draw (\x+0.1, \y+0.5) node[rotate=90] (\yP,\x,left)       {\tiny \( 399.376 \)};
  \draw (\x+0.9, \y+0.5) node[rotate=90] (\yP,\x,right)      {\tiny \( 398.444 \)};

\end{tikzpicture}

  \caption{\label{fig:grid4} This Figure shows estimated values \(Q_{0.99}^{\pol}(s,a)\) of
    Q-Learning after 1 million steps}
\end{figure}

\subsection{Admission Control Queuing System}
\label{subsec:Admission_Control_Queuing_System}

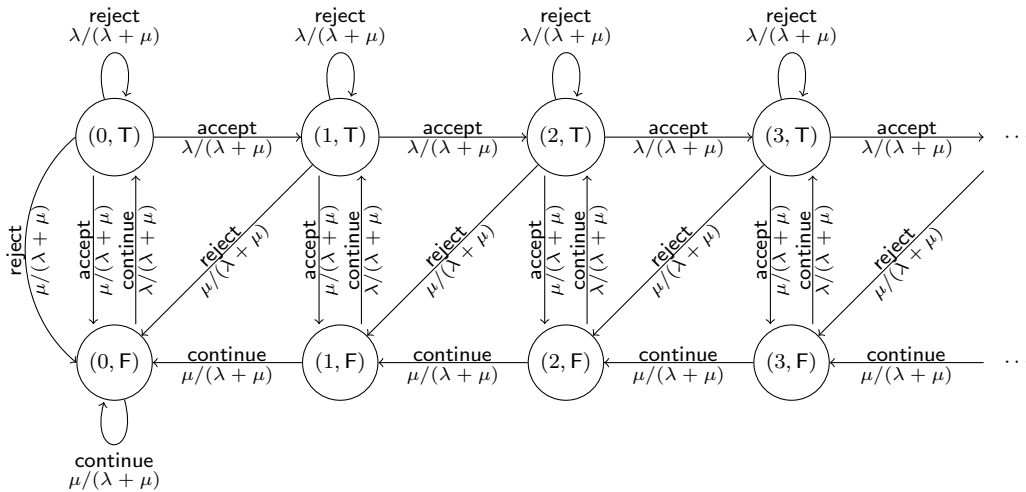
\begin{figure}[t!]
  \centering
  \begin{tikzpicture}[thin, scale=1.0]
  \newcommand\maxX{3}         
  \newcommand\dist{3}
  \newcommand\lastdist{3}
  \newcommand\short{0.405cm}
  \newcommand\move{0.175cm}
  \newcommand\noarrname{continue}
  \tikzmath{\maxXP = int(\maxX - 1); \maxXN = int(\maxX + 1); };

  \foreach \x in {0,...,\maxX} {
    \node[minimum size=25, draw, circle] (arr\x) at (\dist*\x,0) { \((\x,\mathsf{T})\) };
    \node[minimum size=25, draw, circle] (noarr\x) at (\dist*\x,-\dist) { \((\x,\mathsf{F})\) };
  }
  \node[minimum size=25] (arr\maxXN) at (\dist*\maxX+\lastdist,0) { \(\cdots\) };
  \node[minimum size=25] (noarr\maxXN) at (\dist*\maxX+\lastdist,-\dist) { \(\cdots\) };

  \foreach \x in {0,...,\maxX} {
    \path[->] (arr\x) edge[loop above] node[text width=1.5cm,align=center] {\footnotesize
      \textsf{reject}  \\[-0.75ex] \scriptsize \(\lambda / (\lambda + \mu)\)} (arr\x);

    \draw (arr\x) node[right=\move] (arrr\x) {};
    \draw (noarr\x) node[right=\move] (noarrr\x) {};
    \draw (arr\x) node[left=\move] (arrl\x) {};
    \draw (noarr\x) node[left=\move] (noarrl\x) {};
    \path[->,shorten <= \short, shorten >= \short] (noarrr\x) edge node[text
    width=1.5cm,align=center,rotate=90,yshift=-0.75] {\footnotesize  \textsf{\noarrname}
      \\[-0.75ex] \scriptsize \(\lambda / (\lambda + \mu)\)} (arrr\x);
    \path[->,shorten <= \short, shorten >= \short] (arrl\x) edge   node[text
    width=1.5cm,align=center,rotate=90,yshift=-0.75] {\footnotesize  \textsf{accept}  \\[-0.75ex]
      \scriptsize \(\mu / (\lambda + \mu)\)} (noarrl\x);
  }

  \foreach \x in {0,...,\maxX} {
    \tikzmath{\xN = int(\x + 1); };
    \path[thin,->] (arr\x.east) edge  node[text width=1.5cm,align=center,yshift=-1] {\footnotesize
      \textsf{accept}  \\[-0.75ex] \scriptsize \(\lambda / (\lambda + \mu)\)} (arr\xN.west);
  }
  \foreach \x in {0,...,\maxX} {
    \tikzmath{\xN = int(\x + 1); };
    \path[thin,->] (arr\xN) edge  node[text width=1.5cm,align=center,yshift=-1, rotate=45]
    {\footnotesize  \textsf{reject}  \\[-0.75ex] \scriptsize \(\mu / (\lambda + \mu)\)} (noarr\x);
  }

  \foreach \x in {1,...,\maxXN} {
    \tikzmath{\xP = int(\x - 1); };
    \path[thin,->] (noarr\x.west) edge  node[text width=1.5cm,align=center,yshift=-1]
    {\footnotesize  \textsf{\noarrname}  \\[-0.75ex] \scriptsize \(\mu / (\lambda + \mu)\)} (noarr\xP.east);
  }

  \path[->] (noarr0) edge[loop below] node[text width=1.5cm,align=center] {\footnotesize
    \textsf{\noarrname}  \\[-0.75ex] \scriptsize \(\mu / (\lambda + \mu)\)} (noarr0);
  \path[thin,->] (arr0.west) edge[bend right=50]  node[text
  width=1.5cm,align=center,xshift=1,rotate=90] {\footnotesize  \textsf{reject}  \\[-0.2ex]
    \scriptsize \(\mu / (\lambda + \mu)\)} (noarr0.west);
\end{tikzpicture}

  \caption{\label{fig:queuing-system} This diagram illustrates a simple M/M/1 admission control
    queuing system
    (adapted from Mahadevan~\cite{Mahadevan96_AnAveragerewardReinforcementLearningAlgorithmForComputingBiasoptimalPolicies,Mahadevan96_SensitiveDiscountOptimalityUnifyingDiscountedAndAverageRewardReinforcementLearning})
  }
\end{figure}

Finally, like
Mahadevan~\cite{Mahadevan96_AnAveragerewardReinforcementLearningAlgorithmForComputingBiasoptimalPolicies,Mahadevan96_SensitiveDiscountOptimalityUnifyingDiscountedAndAverageRewardReinforcementLearning}
we evaluate the algorithm on a simple M/M/1 admission control queuing system. That is, we assume one
server that processes jobs which arrive by an exponential (Markov) interarrvial time distribution.
Furthermore, the processing duration is assumed to be exponentially distributed.
The arrival and service rate are modeled by parameter \(\lambda\) and \(\mu\) respectively. On each
new arrival the agent has to decide whether to accept the job and thus add it to the queue, or
reject the job. In case of acceptance an immediate reward is received, which however, also incurs a
holding cost depending on the current queue size. The goal is to maximise the reward by balancing
the admission allowance reward and the holding costs. The MDP is depicted in
Figure~\ref{fig:queuing-system} and
was observed through uniformisation from a continuous time problem to a discrete time specification
(see~\cite{Puterman94,bertsekas1995dynamic} for a description of uniformisation). The set of
states consists of elements \((l,\mathsf{Arr})\) with queue length \(l \in \N\) and a Boolean
variable \(\mathsf{Arr} \in \{\mathsf{T},\mathsf{F}\}\) where \(\mathsf{Arr}\) symbolises an arrival
(\(\mathsf{T}\)), or no arrival (\(\mathsf{F}\)).
The edges are labelled with the corresponding action, that is, \textsf{accept} and \textsf{reject} in
case of a new arrival, or \textsf{continue} for continuation when no new job arrived, and the
corresponding transition probability.
We define the reward function \(r\) as
in~\cite{Mahadevan96_AnAveragerewardReinforcementLearningAlgorithmForComputingBiasoptimalPolicies,Mahadevan96_SensitiveDiscountOptimalityUnifyingDiscountedAndAverageRewardReinforcementLearning}
by
\begin{align*}
  r((0,\mathsf{F}), \mathsf{continue}) & = r((0,\mathsf{T}), \mathsf{reject}) = 0 & \text{if } s = 0 \tcom \\
  r((l,\mathsf{F}), \mathsf{continue}) & = -f(l+1)(\lambda + \mu) & \text{if } l \geqslant 1  \tcom\\
  r((l,\mathsf{T}), \mathsf{reject})   & = -f(l+1)(\lambda + \mu) & \text{if } l \geqslant 1 \tcom \\
  r((l,\mathsf{T}), \mathsf{accept})  & = [R - f(l+1)](\lambda + \mu) &
\end{align*}
where the factor \(\lambda + \mu\) is an artifact of the uniformisation of the continuous time
problem to the discrete time MDP.

Haviv and Puterman~\cite{haviv1998bias} show that if the cost function has the shape \(f(l)=c \cdot l\),
there are at most two gain-optimal control limit policies. Namely to admit \(L\) or \(L+1\) jobs.
However, only the policy that admits \(L+1\) jobs is also bias-optimal as the extra reward received
offsets the additional cost of the extra job. Furthermore, note that in such cases the reward
function can be simplified by removing the conditions and the first line. We use exactly this cost
function in our experiment.

Further, we choose the challenging problem setup with \(\lambda=5\), \(\mu=5\), \(R=12\), \(c=1\),
and a maximum queue length of \(20\) as also selected by Mahadevan
in~\cite{Mahadevan96_AnAveragerewardReinforcementLearningAlgorithmForComputingBiasoptimalPolicies,Mahadevan96_SensitiveDiscountOptimalityUnifyingDiscountedAndAverageRewardReinforcementLearning}.
To allow a comparison to the optimal solution we implemented the constraints imposed by constraint
formulation of the addends of the Laurent series
expansion~\cite[p.346]{MillerVeinott1969,Puterman94} using mixed integer linear programming (MILP).
The MILP result shows that \(L=2\), i.e.\@ both policies of admitting \(2\) or \(3\) jobs to the
queue are gain-optimal imposing and average reward of \(\avgrew^{\pol} = 30\). However, only
admitting \(3\) jobs is also bias-optimal and with that Blackwell-optimal, as it's the only
gain-optimal policy that is left. This makes sense, as \(R\) is only collected when an order is
accepted, which is for admitting \(3\) jobs immediate in contrast to the policy of admitting \(2\).
Note that the inferred average reward of about \(27.5\) in
\cite{Mahadevan96_SensitiveDiscountOptimalityUnifyingDiscountedAndAverageRewardReinforcementLearning}
is sub-optimal\footnote{It is possible that the reported results
  in~\cite{Mahadevan96_SensitiveDiscountOptimalityUnifyingDiscountedAndAverageRewardReinforcementLearning}
  do not coincide with the given setup, as the queue lengths do not match our results either. Our
  algorithm infers the policy of admitting \(2\) and queue length \(0.68\) if we fix the average
  reward to \(27.5\) on the above specified setup.}.
The correct queue lengths for admitting \(2\) jobs is \(0.67\), while for \(3\) it is \(1.12\).

For the algorithm setup we use the same values as in the gridworld experiment, except that we
changed the \(\epsilon\)-Parameter to constantly be \(5\) as the returned reward is significantly
higher as in the previous example. As the problem is more complex than the previous ones we
decided to perform \(10^{6}\) learning steps before evaluating for \(100k\) periods.

\begin{table}
  \centering
  \ratab{}
  \begin{tabular}{@{}lrrrr@{}}
    \toprule
     & \multicolumn{2}{c}{Sum Reward}  & \multicolumn{2}{c}{Queue Length}\\
    \cmidrule(r){2-3}\cmidrule(l){4-5}
 Algorithm & Mean & StdDev & Mean & StdDev\\
\midrule
\ARA{} \(\gamma_{1}=1.0\) & \cellcolor{gr}\(\mathbf{2988054.750}\) & \(23977.493\) & \cellcolor{gr}\(1.075\) & \(0.157\)\\
\ARA{} \(\gamma_{1}=0.999\) & \cellcolor{gr}\(2976862.250\) & \(64688.399\) & \cellcolor{gr}\(\mathbf{1.122}\) & \(0.184\)\\
\ARA{} \(\gamma_{1}=0.99\) & \(2683089.250\) & \(1021845.424\) & \(1.545\) & \(1.263\)\\
Q-Learning \(\gamma_{1}=0.99\) & \(45360.750\) & \(13434.404\) & \(0.174\) & \(0.057\)\\
Q-Learning \(\gamma_{1}=0.999\) & \(32609.250\) & \(16216.020\) & \(0.181\) & \(0.080\)\\
Q-Learning \(\gamma_{1}=0.5\) & \(24917.500\) & \(1467.415\) & \(0.002\) & \(0.000\)\\
\bottomrule
  \end{tabular}

\caption{\label{tbl:queuing-system} The admission control queuing system results, where the
  mean inferred reward of \ARA{} with \(\gamma_1\) of \(1.0\), \(0.999\) and \(0.99\) are
  \(29.965\), \(30.272\) and \(28.734\) respectively}

\end{table}

The results are depicted in Table~\ref{tbl:queuing-system}.
The Friedman test rejected the null hypotheses (with \(3.399\text{e}{-32}\) for the mean
reward and \(2.223\text{e}{-35}\) for the mean queue length). We performed the pairwise Conover
post-hoc tests adjusted by the Benjaminyi-Hochberg FDR method~\cite{benjamini1995controlling} and
highlighted measures that are not statistically distinguishable from each other with the a grey
background (see the Appendix for detailed results).
The \ARA{} variants with \(\gamma_{1}=1.0\) and \(\gamma_{1}=0.999\) infer the Blackwell-optimal
policy of admitting \(3\) jobs and thus accumulate a reward of about \(29.88\) and \(29.77\) per
step over all evaluations of the \(40\) replications. This clearly shows the stability of the \ARA{}
algorithm, especially when \(\gamma_{1}=1.0\). However, in this example ARA{} \(\gamma_{1}=0.99\) is
unable to find the optimal policy in \(12\) replications and therefore performs worse as compared to
the other \ARA{} instances. Statistically the optimal queue length of \(1.12\) is matched by \ARA{}
\(\gamma_{1}=0.999\) and \ARA{} \(\gamma_{1}=1.0\).
In contrast all Q-Learning setups are unable to find even close-to-optimal policies, resulting in
rather small amount of collected rewards and very short mean queue lengths.
Furthermore, we investigated how to get Q-Learning \(\gamma_1=0.99\) to find better solutions. Using
\(\epsilon\)-sensitive comparison, instead of the \(\max\)-operator, for the action selection
process we could infer the gain-optimal policy which admits \(2\) jobs. However, we were unable to
infer the Blackwell-optimal policy of admitting \(3\) jobs and thus collects rewards as soon as
possible with Q-Learning.
Similarly by hyperparameter tuning of \ARA{} \(\gamma_{1}=1.0\), namely increasing the decay rates
of \(\alpha\) and \(\gamma\) to \(0.8\) and omitting the minimum values, we were able to find more
stable state-action values such that setting \(\epsilon \leqslant 1\) is possible while still
finding the optimal policy and the average reward stabilizes even more. We found that the stability
of the average reward is of major importance for the stability our \ARA{} algorithm.

\section{Conclusion}%
\label{sec:Conclusion}


This paper introduces deep theoretical insights to reinforcement learning and explains why standard
discounted reinforcement learning is inappropriate for tasks that are presented with rewards in
non-terminal steps also as they produce an average reward per step that cannot be approximated with
\(0\). This kind of problem structure is easily obtained in real-world problem specifications, for
instance in the field of operations research, where companies constantly aim for profit-optimal
decisions. Furthermore, we established a novel average reward adjusted discounted reinforcement
learning algorithm \ARA{}, which is computationally cheap and deduces near-Blackwell-optimal
policies. Additionally, we implemented the algorithm and prove its viability by testing it on three
different decision problems. The results experimentally expose the superiority of \ARA{} over
standard discounted reinforcement learning.
In the future we plan to use neural networks for function approximation to be able to apply \ARA{}
to bigger sized problems. Adding to this we are planning to develop an actor-critic version of
\ARA{}.

Therefore, in conclusion although machine learning has tremendously advanced and solved many
problems in many different areas the methods may not be directly applicable to operations research.
Thus, as a researcher it is important to have broad background knowledge of the selected method
instead of handling it as a black box, as we showed that sometimes the problem structures require an
adaption of the machine learning technique for successful applications. Nonetheless, machine
learning and in particular reinforcement learning have been producing astonishing results over the
past decades, which when handled wisely can be adapted to our field of research as we demonstrated
by establishing \ARA{}. Therefore, machine learning methods will likely play an important role in
advancing the field of operations research and its techniques over the next years, especially when
used in combination with theoretical insights of our field.

\bibliographystyle{spmpsci}      
\bibliography{references}

\begin{thebibliography}{10}
\providecommand{\url}[1]{{#1}}
\providecommand{\urlprefix}{URL }
\expandafter\ifx\csname urlstyle\endcsname\relax
  \providecommand{\doi}[1]{DOI~\discretionary{}{}{}#1}\else
  \providecommand{\doi}{DOI~\discretionary{}{}{}\begingroup
  \urlstyle{rm}\Url}\fi

\bibitem{balaji2019orl}
Balaji, B., Bell-Masterson, J., Bilgin, E., Damianou, A., Garcia, P.M., Jain,
  A., Luo, R., Maggiar, A., Narayanaswamy, B., Ye, C.: Orl: Reinforcement
  learning benchmarks for online stochastic optimization problems.
\newblock arXiv preprint arXiv:1911.10641  (2019)

\bibitem{bello2016neural}
Bello, I., Pham, H., Le, Q.V., Norouzi, M., Bengio, S.: Neural combinatorial
  optimization with reinforcement learning.
\newblock arXiv preprint arXiv:1611.09940  (2016)

\bibitem{benjamini1995controlling}
Benjamini, Y., Hochberg, Y.: Controlling the false discovery rate: a practical
  and powerful approach to multiple testing.
\newblock Journal of the Royal statistical society: series B (Methodological)
  \textbf{57}(1), 289--300 (1995)

\bibitem{bertsekas1995dynamic}
Bertsekas, D.P., Bertsekas, D.P., Bertsekas, D.P., Bertsekas, D.P.: Dynamic
  programming and optimal control, vol.~1.
\newblock Athena scientific Belmont, MA (1995)

\bibitem{Blackwell62}
Blackwell, D.: Discrete dynamic programming.
\newblock The Annals of Mathematical Statistics \textbf{344}, 719--726 (1962).
\newblock \doi{016/j.cam.2018.05.030}

\bibitem{boyan1995generalization}
Boyan, J.A., Moore, A.W.: Generalization in reinforcement learning: Safely
  approximating the value function.
\newblock In: Advances in neural information processing systems, pp. 369--376
  (1995)

\bibitem{chaharsooghi2008reinforcement}
Chaharsooghi, S.K., Heydari, J., Zegordi, S.H.: A reinforcement learning model
  for supply chain ordering management: An application to the beer game.
\newblock Decision Support Systems \textbf{45}(4), 949--959 (2008)

\bibitem{das1999solving}
Das, T.K., Gosavi, A., Mahadevan, S., Marchalleck, N.: Solving semi-markov
  decision problems using average reward reinforcement learning.
\newblock Management Science \textbf{45}(4), 560--574 (1999)

\bibitem{enns2004work}
Enns, S.T., Suwanruji, P.: Work load responsive adjustment of planned lead
  times.
\newblock Journal of Manufacturing Technology Management \textbf{15}(1),
  90--100 (2004)

\bibitem{gijsbrechts2018can}
Gijsbrechts, J., Boute, R.N., Van~Mieghem, J.A., Zhang, D.: Can deep
  reinforcement learning improve inventory management? performance and
  implementation of dual sourcing-mode problems.
\newblock Performance and Implementation of Dual Sourcing-Mode Problems
  (December 17, 2018)  (2018)

\bibitem{haviv1998bias}
Haviv, M., Puterman, M.L.: Bias optimality in controlled queueing systems.
\newblock Journal of Applied Probability \textbf{35}(1), 136--150 (1998)

\bibitem{hax1973hierarchical}
Hax, A.C., Meal, H.C.: Hierarchical integration of production planning and
  scheduling.
\newblock Report, DTIC Document (1973)

\bibitem{howard1960dynamic}
Howard, R.A.: Dynamic programming and markov processes.
\newblock John Wiley (1960)

\bibitem{kool2018attention}
Kool, W., van Hoof, H., Welling, M.: Attention, learn to solve routing
  problems!
\newblock arXiv preprint arXiv:1803.08475  (2018)

\bibitem{Lillicrap15}
Lillicrap, T.P., Hunt, J.J., Pritzel, A., Heess, N., Erez, T., Tassa, Y.,
  Silver, D., Wierstra, D.: Continuous control with deep reinforcement
  learning.
\newblock arXiv preprint arXiv:1509.02971  (2015)

\bibitem{Mahadevan96_AnAveragerewardReinforcementLearningAlgorithmForComputingBiasoptimalPolicies}
Mahadevan, S.: An average-reward reinforcement learning algorithm for computing
  bias-optimal policies.
\newblock In: AAAI/IAAI, Vol. 1, pp. 875--880 (1996)

\bibitem{Mahadevan96_AverageRewardReinforcementLearningFoundationsAlgorithmsAndEmpiricalResults}
Mahadevan, S.: Average reward reinforcement learning: Foundations, algorithms,
  and empirical results.
\newblock Machine Learning \textbf{22}, 159--195 (1996)

\bibitem{Mahadevan96_OptimalityCriteriaInReinforcementLearning}
Mahadevan, S.: Optimality criteria in reinforcement learning.
\newblock In: Proceedings of the AAAI Fall Symposium on Learning Complex
  Behaviors in Adaptive Intelligent Systems (1996)

\bibitem{Mahadevan96_SensitiveDiscountOptimalityUnifyingDiscountedAndAverageRewardReinforcementLearning}
Mahadevan, S.: Sensitive discount optimality: Unifying discounted and average
  reward reinforcement learning.
\newblock In: ICML, pp. 328--336 (1996)

\bibitem{mahadevan1997self}
Mahadevan, S., Marchalleck, N., Das, T.K., Gosavi, A.: Self-improving factory
  simulation using continuous-time average-reward reinforcement learning.
\newblock In: Machine Learning-International Workshop Then Conference-, pp.
  202--210. Morgan Kaufmann Publishers, Inc. (1997)

\bibitem{mahadevan1998optimizing}
Mahadevan, S., Theocharous, G.: Optimizing production manufacturing using
  reinforcement learning.
\newblock In: FLAIRS Conference, pp. 372--377 (1998)

\bibitem{MillerVeinott1969}
Miller, B.L., Veinott, A.F.: Discrete dynamic programming with a small interest
  rate.
\newblock The Annals of Mathematical Statistics \textbf{40}(2), 366--370
  (1969).
\newblock \urlprefix\url{http://www.jstor.org/stable/2239451}

\bibitem{mnih2016asynchronous}
Mnih, V., Badia, A.P., Mirza, M., Graves, A., Lillicrap, T., Harley, T.,
  Silver, D., Kavukcuoglu, K.: Asynchronous methods for deep reinforcement
  learning.
\newblock In: International conference on machine learning, pp. 1928--1937
  (2016)

\bibitem{Mnih15_HumanlevelControlThroughDeepReinforcementLearningb}
Mnih, V., Kavukcuoglu, K., Silver, D., Rusu, A.A., Veness, J., Bellemare, M.G.,
  Graves, A., Riedmiller, M., Fidjeland, A.K., Ostrovski, G., et~al.:
  Human-level control through deep reinforcement learning.
\newblock Nature \textbf{518}(7540), 529 (2015)

\bibitem{mnih2015human}
Mnih, V., Kavukcuoglu, K., Silver, D., Rusu, A.A., Veness, J., Bellemare, M.G.,
  Graves, A., Riedmiller, M., Fidjeland, A.K., Ostrovski, G., et~al.:
  Human-level control through deep reinforcement learning.
\newblock Nature \textbf{518}(7540), 529 (2015)

\bibitem{nazari2018reinforcement}
Nazari, M., Oroojlooy, A., Snyder, L., Tak{\'a}c, M.: Reinforcement learning
  for solving the vehicle routing problem.
\newblock In: Advances in Neural Information Processing Systems, pp. 9839--9849
  (2018)

\bibitem{ok1996auto}
Ok, D., Tadepalli, P.: Auto-exploratory average reward reinforcement learning.
\newblock In: AAAI/IAAI, Vol. 1, pp. 881--887 (1996)

\bibitem{oroojlooyjadid2017deep}
Oroojlooyjadid, A., Nazari, M., Snyder, L., Tak{\'a}{\v{c}}, M.: A deep
  q-network for the beer game: A reinforcement learning algorithm to solve
  inventory optimization problems.
\newblock arXiv preprint arXiv:1708.05924  (2017)

\bibitem{Puterman94}
Puterman, M.L.: Markov decision processes. j.
\newblock Wiley and Sons  (1994)

\bibitem{rohde2004hierarchical}
Rohde, J.: Hierarchical supply chain planning using artificial neural networks
  to anticipate base-level outcomes.
\newblock OR Spectrum \textbf{26}(4), 471--492 (2004)

\bibitem{schneckenreither2018reinforcement}
Schneckenreither, M., Haeussler, S.: Reinforcement learning methods for
  operations research applications: The order release problem.
\newblock In: International Conference on Machine Learning, Optimization, and
  Data Science, pp. 545--559. Springer (2018)

\bibitem{schwartz1993reinforcement}
Schwartz, A.: A reinforcement learning method for maximizing undiscounted
  rewards.
\newblock In: Proceedings of the tenth international conference on machine
  learning, vol. 298, pp. 298--305 (1993)

\bibitem{schwartz1993thinking}
Schwartz, A.: Thinking locally to act globally: A novel approach to
  reinforcement learning.
\newblock In: Proceedings of the fifteenth annual conference of the cognitive
  science society, pp. 906--911. Lawrence Erlbaum Associates Hillsdale, NJ
  (1993)

\bibitem{Silver16_MasteringTheGameOfGoWithDeepNeuralNetworksAndTreeSearch}
Silver, D., Huang, A., Maddison, C.J., Guez, A., Sifre, L., Van Den~Driessche,
  G., Schrittwieser, J., Antonoglou, I., Panneershelvam, V., Lanctot, M.,
  et~al.: Mastering the game of go with deep neural networks and tree search.
\newblock nature \textbf{529}(7587), 484 (2016)

\bibitem{Silver17_MasteringChessAndShogiBySelfPlayWithAGeneralReinforcementLearningAlgorithm}
Silver, D., Hubert, T., Schrittwieser, J., Antonoglou, I., Lai, M., Guez, A.,
  Lanctot, M., Sifre, L., Kumaran, D., Graepel, T., et~al.: Mastering chess and
  shogi by self-play with a general reinforcement learning algorithm.
\newblock arXiv preprint arXiv:1712.01815  (2017)

\bibitem{sutton1998introduction}
Sutton, R.S., Barto, A.G., et~al.: Introduction to reinforcement learning,
  vol.~2.
\newblock MIT press Cambridge (1998)

\bibitem{Tadepalli98_ModelbasedAverageRewardReinforcementLearning}
Tadepalli, P., Ok, D.: Model-based average reward reinforcement learning.
\newblock Artificial intelligence \textbf{100}(1-2), 177--224 (1998)

\bibitem{Veinott69}
Veinott, A.F.: Discrete dynamic programming with sensitive discount optimality
  criteria.
\newblock The Annals of Mathematical Statistics \textbf{40}(5), 1635--1660
  (1969).
\newblock \doi{10.1214/aoms/1177697379}

\bibitem{vera2019deep}
Vera, J.M., Abad, A.G.: Deep reinforcement learning for routing a heterogeneous
  fleet of vehicles.
\newblock arXiv preprint arXiv:1912.03341  (2019)

\bibitem{Watkins89_LearningFromDelayedRewards}
Watkins, C.J.C.H.: Learning from delayed rewards.
\newblock Ph.D. thesis, King's College (1989)

\bibitem{zhang1995reinforcement}
Zhang, W., Dietterich, T.G.: A reinforcement learning approach to job-shop
  scheduling.
\newblock In: IJCAI, vol.~95, pp. 1114--1120. Citeseer (1995)

\bibitem{zijm2000towards}
Zijm, W.H.: Towards intelligent manufacturing planning and control systems.
\newblock OR-Spektrum \textbf{22}(3), 313--345 (2000)

\end{thebibliography}

\vfill
\pagebreak
\appendix

\section{Statistical Results}%
\label{sec:statsistical_results}

This section provides the detailed results of the pairwise Benjaminyi-Hochberg FDR adjusted Conover
statistical analysis for the the gridworld and the admission control queuing system examples.

\subsection{Statistical Significance for the Gridworld Example}

\begin{table}[h!]
  \centering
  \ratab{}
    \begin{tabular}{@{}lrrrrr@{}}
    \toprule
      & \multicolumn{3}{c}{\ARA{}}  & \multicolumn{2}{c}{Q-Learning}\\
    \cmidrule(r){2-4}\cmidrule(l){5-6}
 & \(\gamma_{1}=0.99\) & \(\gamma_{1}=0.999\) & \(\gamma_{1}=1.0\) & \(\gamma_{1}=0.99\) & \(\gamma_{1}=0.999\)\\
\midrule
\ARA{} \(\gamma_1=0.999\) & 6.723e-01 &  &  &  & \\
\ARA{} \(\gamma_1=1.0\) & 7.834e-01 & 7.834e-01 &  &  & \textbf{Mean}\\
Q-Learning  \(\gamma_1=0.99\) & 2.048e-71 & 3.365e-70 & 8.198e-71 &  & \textbf{Sum Rew.}\\
Q-Learning  \(\gamma_1=0.999\) & 1.094e-75 & 1.361e-74 & 3.764e-75 & 6.936e-02 & \\
Q-Learning  \(\gamma_1=0.50\) & 3.893e-96 & 1.392e-95 & 6.353e-96 & 2.333e-25 & 1.148e-19\\
\midrule
\ARA{} \(\gamma_1=0.999\) & 2.465e-06 &  &  &  & \textbf{Mean}\\
\ARA{} \(\gamma_1=1.0\) & 5.249e-04 & 1.957e-01 &  &  & \textbf{Avg. S.}\\
Q-Learning  \(\gamma_1=0.99\) & 3.083e-89 & 3.696e-79 & 5.691e-82 &  & \textbf{to Goal}\\
Q-Learning  \(\gamma_1=0.999\) & 7.900e-88 & 1.531e-77 & 2.041e-80 & 4.583e-01 & \\
Q-Learning  \(\gamma_1=0.50\) & 2.242e-99 & 1.352e-90 & 4.932e-93 & 7.775e-08 & 1.787e-09\\
    \bottomrule
\end{tabular}
\caption{\label{tbl:grid-stats} The Benjaminyi-Hochberg FDR adjusted Conover p-values
  for the mean sum of rewards (top) and the mean average number of steps to the goal state (bottom)}
\end{table}

\subsection{Statistical Significance for the Admission Control Queue System Example}

\begin{table}[h!]
  \centering
  \ratab{}
  \begin{tabular}{@{}lrrrrr@{}}
    \toprule
     & \multicolumn{3}{c}{\ARA{}}  & \multicolumn{2}{c}{Q-Learning}\\
    \cmidrule(r){2-4}\cmidrule(l){5-6}
 & \(\gamma_{1}=1.0\) & \(\gamma_{1}=0.999\) & \(\gamma_{1}=0.99\) & \(\gamma_{1}=0.99\) & \(\gamma_{1}=0.999\)\\
\midrule
\ARA{} \(\gamma_{1}=0.999\) & 5.947e-01 &  &  &  & \\
\ARA{} \(\gamma_{1}=0.99\) & 7.002e-11 & 3.469e-12 &  &  & \textbf{Mean}\\
Q-Learning \(\gamma_{1}=0.99\) & 6.035e-64 & 3.056e-65 & 5.232e-45 &  & \textbf{Sum Rew.}\\
Q-Learning \(\gamma_{1}=0.999\) & 1.509e-83 & 1.498e-84 & 7.776e-68 & 5.639e-15 & \\
Q-Learning \(\gamma_{1}=0.5\) & 5.459e-94 & 1.061e-94 & 3.485e-80 & 1.154e-30 & 2.906e-07\\
\midrule
\ARA{} \(\gamma_{1}=0.999\) & 6.556e-01 &  &  &  & \\
\ARA{} \(\gamma_{1}=0.99\) & 1.554e-05 & 2.370e-06 &  &  & \textbf{Mean}\\
Q-Learning \(\gamma_{1}=0.99\) & 4.527e-74 & 4.329e-73 & 7.764e-84 &  & \textbf{Queue Len.}\\
Q-Learning \(\gamma_{1}=0.999\) & 7.425e-69 & 8.293e-68 & 4.425e-79 & 2.858e-02 & \\
Q-Learning \(\gamma_{1}=0.5\) & 1.606e-105 & 6.136e-105 & 1.763e-112 & 1.870e-37 & 4.645e-44\\
    \bottomrule
\end{tabular}

\caption{\label{tbl:queuing-system} The Benjaminyi-Hochberg FDR adjusted Conover p-values for the
  mean sum reward (top) and mean queue length (bottom)}

\end{table}

\end{document}